\documentclass{article}

\PassOptionsToPackage{numbers}{natbib}

    \usepackage[preprint]{neurips_2025}

\usepackage[utf8]{inputenc} % allow utf-8 input
\usepackage[T1]{fontenc}    % use 8-bit T1 fonts
\usepackage{hyperref}       % hyperlinks
\usepackage{url}            % simple URL typesetting
\usepackage{booktabs}       % professional-quality tables
\usepackage{amsfonts}       % blackboard math symbols
\usepackage{nicefrac}       % compact symbols for 1/2, etc.
\usepackage{microtype}      % microtypography
\usepackage{xcolor}         % colors
\usepackage{wrapfig} 

\usepackage{amsmath}
\usepackage{amssymb}
\usepackage{mathtools}
\usepackage{amsthm}

\usepackage[capitalize,noabbrev]{cleveref}

\usepackage{graphicx}
\usepackage{subfigure}
\usepackage{multirow}
\usepackage{colortbl}

\theoremstyle{plain}
\newtheorem{theorem}{Theorem}[section]

\newtheorem{lemma}[theorem]{Lemma}

\theoremstyle{definition}

\theoremstyle{remark}

\title{Unsupervised Dynamic Feature Selection for Robust Latent Spaces in Vision Tasks
 }

\author{%
  Bruno Corcuera\\
  Universidade da Coruña\\
  A Coruña, Spain\\
  \texttt{bruno.sanchez1@udc.es}
  \And
  Carlos Eiras-Franco\\
  CITIC Research Center\\
  Universidade da Coruña\\
  A Coruña, Spain\\
  \texttt{carlos.eiras.franco@udc.es}
  \And
  Brais Cancela\\
  CITIC Research Center\\
  Universidade da Coruña\\
  A Coruña, Spain\\
  \texttt{brais.cancela@udc.es}
}

\begin{document}

\maketitle

\begin{abstract}
Latent representations are critical for the performance and robustness of machine learning models, as they encode the essential features of data in a compact and informative manner. However, in vision tasks, these representations are often affected by noisy or irrelevant features, which can degrade the model’s performance and generalization capabilities. This paper presents a novel approach for enhancing latent representations using unsupervised Dynamic Feature Selection (DFS). For each instance, the proposed method identifies and removes misleading or redundant information in images, ensuring that only the most relevant features contribute to the latent space. By leveraging an unsupervised framework, our approach avoids reliance on labeled data, making it broadly applicable across various domains and datasets. Experiments conducted on image datasets demonstrate that models equipped with unsupervised DFS achieve significant improvements in generalization performance across various tasks, including clustering and image generation, while incurring a minimal increase in the computational cost.%This contribution highlights the potential of dynamic feature selection to enhance the overall performance of machine learning models.

\end{abstract}

\section{Introduction}
\label{sec:introduction}

Feature selection and feature extraction are two essential techniques in machine learning and data analysis, both aimed at improving the efficiency and effectiveness of predictive modeling \cite{zebari2020comprehensive}. While both approaches share the goal of reducing the dimensionality of datasets, they diverge significantly in their fundamental strategies and objectives. 

Feature selection aims to identify a subset of relevant features from a larger set of available features that are most informative for a given task \cite{dhal2022comprehensive}. Traditional feature selection methods typically rely on static criteria, selecting a fixed set of features during model training \cite{balin2019concrete,roffo2020infinite,cancela2023e2e}. However, in dynamic and evolving data environments, where the relevance of features may change over different samples, static feature selection may not be optimal. Dynamic Feature Selection (DFS) (also referred as \textit{Instance-based or Instancewise Feature Selection} \cite{yoon2018invase,panda2021instance,liyanage2021dynamic}) represents a paradigm shift in feature selection by recognizing the variability of feature importance \cite{chen2018learning,arik2021tabnet} between each sample. Unlike traditional static methods, DFS algorithms adaptively adjust the feature subset during model training or deployment, accommodating changes in data characteristics and task requirements.

Moreover, feature extraction techniques transform the original feature space into a new space by creating a set of derived features that capture essential information from the original data. These methods, such as Principal Component Analysis (PCA), Linear Discriminant Analysis (LDA), or any deep features derived from a deep learning model, seek to maximize the discriminative power or variance of the transformed features \cite{perera2019learning,tang2022multiscale,izmailov2022feature}. In doing so, they often create a smaller, more compact representation of the data, potentially enhancing model performance but hindering its interpretability.

While both feature extraction and feature selection enhance the quality of input data for machine learning models, they differ fundamentally in their approach. Feature extraction generates entirely new features, potentially altering the interpretability of the data, while feature selection retains the original features, preserving the original meaning and context. DFS integrates feature extraction's versatility with the interpretability of feature selection. No prior DFS method works without labels and preserves 2‑D structure – DDS fills that gap.

\begin{figure*}[ht]
	\centering
	\begin{minipage}[]{1.0\linewidth}
	    \centering
 		\includegraphics[width=1.0\textwidth]{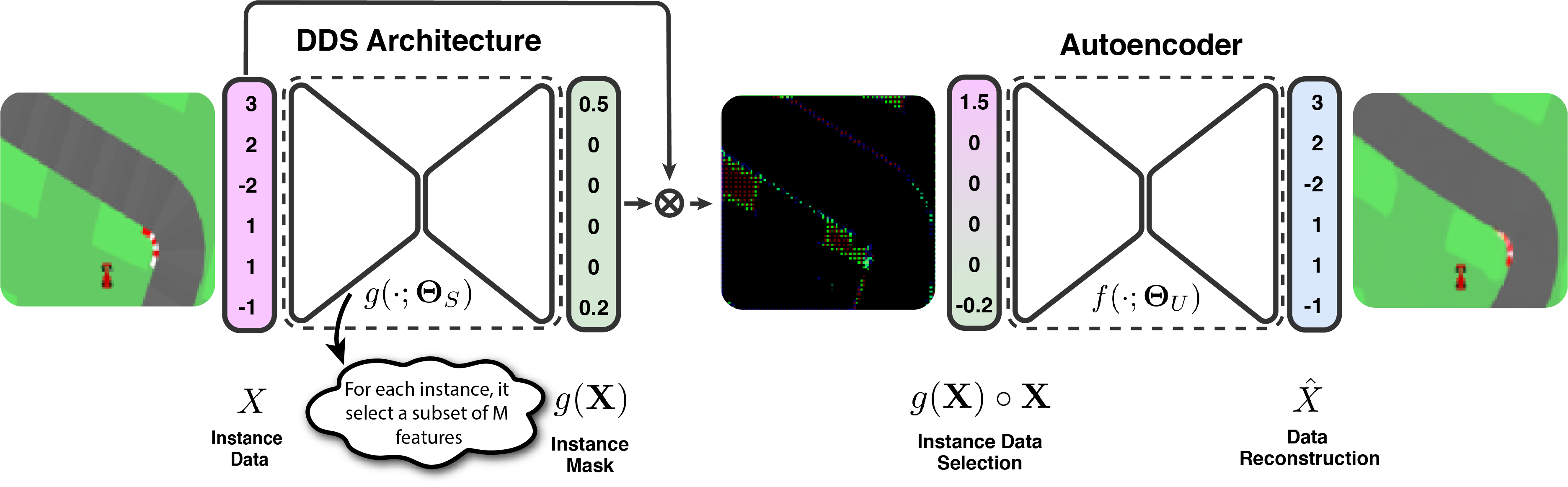}\\
	\end{minipage}
\caption{
 Proposed method. The DDS module is prepended to an existing architecture (an autoencoder in this case), substituting its input with an equally shaped masked version that retains only the most relevant features. DDS is in charge of selecting, for each sample, the most relevant features for the downstream architecture to solve the unsupervised task (in this particular example, data reconstruction).
}
\label{fig:architecture}
\end{figure*}

Thus, this work presents a novel DFS method that aims to improve the latent representations by removing information that is either irrelevant or leads to erroneous interpretations (i.e. background elements), which may appear in different locations depending on the input data. The main contributions of this paper are the following:

\begin{itemize}
    \item We present Dynamic Data Selection (DDS), a novel unsupervised dynamic feature selection algorithm. To our knowledge, this is the first attempt to provide a DFS solution for unsupervised scenarios.
    \item Contrary to previous supervised approaches like \cite{chen2018learning}, DDS's memory consumption is minimal, and invariant to the maximum number of selected features.
    \item DDS can be easily adapted to a variety of problems and architectures, and it enables the architect designer to use complex networks since it preserves the position of the selected features. Fig. \ref{fig:architecture} shows, as an example, how to attach the DDS architecture to solve a data reconstruction task.
    %\item A variation to the hard concrete distribution \cite{louizos2018learning} will be presented. Taking the properties of the DDS algorithm, it will be used during the training procedure to minimize the chances of premature feature eliminations.
    \item We report extensive tests of DDS in two different unsupervised scenarios: clustering, and representation learning for world models. These experiments show DDS's adaptability to enhance the quality of latent representations, as well as its ease of use.
\end{itemize}

%This paper is organized as follows: Section \ref{sec:related_work} introduces the most important DFS approaches; Section \ref{sec:formulation} describes the novel problem formulation, our proposed method, and details regarding the implementation and the training procedure; Section \ref{sec:results} reports experimental results about the ability of this dynamic feature selection approach to compress the input data to be used in a different task, like clustering; and finally, Section \ref{sec:discussion} lists our conclusions and future work ideas.

\section{Related Work}
\label{sec:related_work}

DFS is a recent field of study, with almost no contributions before the rise of deep learning architectures. Four works stand out among the rest: \textit{Learning to Explain} (L2X) \cite{chen2018learning}, \textit{INVASE} \cite{yoon2018invase}, TabNet \cite{arik2021tabnet} and LSPIN \cite{yang2022locally}. However, it is worth noting that these algorithms were entirely developed for supervised learning. %, that is, they require the supervising signal of the task to select the relevant features. % HABLAR DE LSPIN \cite{yang2022locally}

L2X \cite{chen2018learning} (and variants like Greedy \cite{covert2023learning}) consists of an autoencoder-like architecture that is attached before the classification model. Its output is of the form $\mathcal{R}^{N \times F \times M}$, with $M$ being the maximum number of features to be selected. Then, each input sample is transformed by performing a matrix multiplication. Although this solution provides remarkable results in supervised scenarios, it suffers from two major drawbacks: first, it has high memory requirements, as the output size of the model is dependent on the number of maximum features to be selected, forcing $M$ to have smaller values; and second, the matrix multiplication procedure forces the input data to be one dimensional, preventing the use of complex layers like 2D convolutions in the classification model.

INVASE \cite{yoon2018invase} consists of 3 different networks: a selector, a predictor, and a baseline. Inspired by the actor-critic method \cite{peters2008natural}, the predictor and the baseline output classification scores but using different input data: the baseline receives all input features, but the predictor only uses a small subset of the input features. The selector algorithm determines the most suitable subset of input features. Compared to L2X, the main advantage of this algorithm is that it can preserve the spatial information of the input data. A major drawback is that it cannot be adapted to unsupervised scenarios, since it requires supervised data to train.

Moreover, TabNet \cite{arik2021tabnet,shah2022enhanced} and LSPIN \cite{yang2022locally} use sequential attention to select the most important features per sample. They were specially designed to be used in microarray scenarios, where the number of features is far greater than the number of samples. The main advantage of these approaches is their inherent explainability: they facilitate investigating the relevant features for each classification. However, their computational cost is too high for big data environments, and their accuracy is lower than L2X and INVASE.

\section{Problem Formulation}
\label{sec:formulation}

Let $\mathbf{X} \in \mathcal{R}^{N \times F}$ be our input data, where $N$ is the number of instances and $F$ is the total number of different features. We aim to select, for each instance, a maximum number of features, denoted by $M$. Formally speaking, our algorithm aims to solve the following minimization problem:
\begin{equation}
    \label{eq:formulation}
    \begin{aligned}
    & \underset{\mathbf{\Theta}_S, \mathbf{\Theta}_U}{\text{minimize}}
    & & \mathcal{L}(f(g(\mathbf{X}; \mathbf{\Theta}_S) \circ \mathbf{X}; \mathbf{\Theta}_U)) \\
    & \text{subject to}
    & & g(\mathbf{X}; \mathbf{\Theta}_S) \in [0, 1]^{N \times F}, \\
    & 
    & & \| g(\mathbf{X}; \mathbf{\Theta}_S)^{(i)} \|_0 \leq M, \forall i \in \{1..N\}.
    % & & \sum_{i=1}^{|\mathbf{\gamma}|} \gamma_i  \leq M.
    \end{aligned}
\end{equation}

where $\mathcal{L}$ is the unsupervised loss function, $M$ is the maximum number of features to be selected, $g(\cdot; \mathbf{\Theta}_S)$ is the DDS network, and $f(\cdot; \mathbf{\Theta}_U)$ is the unsupervised task to solve. For the sake of simplicity, initially $f(\cdot; \mathbf{\Theta}_U)$ will be considered as an autoencoder that aims to reconstruct the initial features, including those that are masked by $g(\cdot; \mathbf{\Theta}_S)$. Section \ref{sec:results} shows how this approach can be adapted to solve other unsupervised tasks. The key idea is simple: an autoencoder architecture ($f$) is trained while introducing an extra network ($g$), tasked with selecting, at most, $M$ relevant features per sample and masking the rest. After that, those features will be passed to the autoencoder model to reconstruct the whole unmasked input data.

% It is worth noting that this is not the first attempt to create a Dynamic Feature Selection architecture. This idea was previously addressed in \cite{chen2018learning} for a supervised problem. The output of the architecture is of the form $\mathcal{R}^{N \times F \times M}$, being $M$ the maximum number of features to be selected. Then, each input sample is transformed by performing a matrix multiplication. Although this solution provides remarkable results in supervised scenarios, it suffers from two major drawbacks: first, the variable memory requirement, as the output size of the model is dependent on the number of maximum features to be selected, forcing $M$ to have smaller values; the second drawback is related to the matrix multiplication procedure, as it forces to input data to be one dimensional.

In this paper, we introduce the DDS architecture, with a $\mathcal{R}^{N \times F}$  output. The size of this output is not dependent on the maximum number of selected features. Note that this method is easily applicable to existing architectures solving unsupervised problems, as long as they can be trained by using gradient descent. DDS is defined as a module to be inserted before the existing architecture, replacing its input with masked input data. After training the whole architecture, the output of the DDS layer will contain only the most relevant features of each input. Please refer to Appendix \ref{sec:proof} for an intuition of why this improves the representation.

\subsection{DDS Implementation}
\label{sec:DDS_implementation}

As depicted in Eq. \ref{eq:formulation}, two constraints have to be addressed to solve the minimization problem. 
% Although this formulation is very similar to a recent feature selection algorithm called E2E-FS \cite{cancela2023e2e}, the solution proposed by the authors cannot be applied in this particular dynamic approach. E2E-FS aims to select the same $M$ features for all instances, forcing the discarded ones to be zeroed. Although this is a good solution for a static feature selection approach, it is too restrictive for a dynamic version. We aim to keep the scores of the discarded features, so they can be used for other purpose, like model explanation. 

To get a differentiable approach, we propose to reformulate the problem as
\begin{equation}
    \label{eq:reformulation}
    \begin{aligned}
    & \underset{\mathbf{\Theta}_S, \mathbf{\Theta}_U}{\text{minimize}}
    & & \mathcal{L}(f(\tau(\tilde{g}(\mathbf{X}; \mathbf{\Theta}_S) + \delta) \odot \mathbf{\Gamma}_M \odot \mathbf{X}; \mathbf{\Theta}_U)) \\ %+ \alpha \mathcal{L}_0(\tilde{g}(\mathbf{X}; \mathbf{\Theta}_S)) \\
    & \text{subject to}
    & & \tau(\tilde{g}(\mathbf{X}; \mathbf{\Theta}_S) + \delta) \in [0, 1]^{N \times F}, \\
    & 
    & & \mathbf{\Gamma}_M \in \{0, 1\}^{N \times F}, \\
    & 
    & & \| \mathbf{\Gamma}_M \|_0 \leq M, \forall i \in \{1..N\}.
    % & & \sum_{i=1}^{|\mathbf{\gamma}|} \gamma_i  \leq M.
    \end{aligned}
\end{equation}
where
\begin{equation}
\label{eq:hard_concrete}
    \tau(x) = \min \left(1, \max \left(0, \sigma \left(\dfrac{x}{\beta} \right)(\zeta - \gamma) + \gamma \right) \right),
\end{equation}
where $\sigma$ is the sigmoid function. $\tau(x)$ is the hard concrete gate presented in \cite{louizos2018learning}, and $\delta$ is a hyper-parameter to ensure non-zero values at the beginning of the training.

The idea is to split the DDS model $g$ into two different terms. In the first place, $\tilde{g}(\cdot; \mathbf{\Theta}_S)$ is implemented in the same way as $g$ was previously defined, without the 0-norm constraint. This constraint is now placed to a different matrix, called $\mathbf{\Gamma}_M$. $\mathbf{\Gamma}_M$ is a binary matrix with all zeroes but the top-M scores of each sample of $\tilde{g}(\cdot; \mathbf{\Theta}_S)$.  By default, the hyper-parameters are set to $\beta = \tfrac{2}{3}$, $\delta = 1$, $\zeta = 1$ and $\gamma = 0$.

\subsection{Hacking the Training Procedure}
\label{sec:training}

The training procedure is similar to the one used in \cite{louizos2018learning}, with three slight variations.

First, the $l_0$-norm regularization term is not included. As the $\Gamma$ matrix already removes $F - M$ features, it is unnecessary to use a regularization term to force some features to drop their importance to $0$. For the same reason, $\gamma < 0$ values are not needed. Our ablation study (Section \ref{sec:ablation}) shows that using the default $\zeta = 1.1$ and $\gamma = -0.1$ values provided in \cite{louizos2018learning} results in a performance drop in the performance.

Second, an additional hyperparameter is used to control the binary concrete distribution. The hard binary concrete distribution presented in \cite{louizos2018learning} aims to increase the probability mass near 0 and 1, to either force some features to be discarded, or to increase the feature probability to near 1. Using the original distribution is also counterproductive, as it could unadvisedly force the model to remove more features than the target $M$, causing a degradation in performance. A variation of this distribution is presented to avoid this problem, and a hyperparameter is included in the formulation. This distribution is defined as
\begin{equation}
    \label{eq:distribution}
    \tau_u(x) = \min \left(1, \max \left(0, \tilde{\sigma}_u(x)(\zeta - \gamma) + \gamma \right) \right),
\end{equation}
where
\begin{equation}
    \tilde{\sigma}_u(x) = \sigma \left(\dfrac{\kappa(\log(u) - \log(1 - u)) + x}{\beta} \right), ~u \in \mathcal{U}(0, 1),
\end{equation}
and $\kappa \in [0, 1]$. By default, $\kappa$ is set to $0.1$ during the training procedure, and $0$ for testing.

Finally, early experiments suggest that the algorithm struggles with initialization when using low $M$ values ($M << F$). To solve this problem, $M$ is dynamically changed, during training, by using the following variation:
\begin{equation}
    \label{eq:dynamic_M}
    \begin{aligned}
    M^{(i)}_t &= \begin{cases}
                        M \quad p^{(i)}_t > \epsilon, \quad p^{(i)}_t \in~ \mathcal{U}(0, 1) \\
                        F \quad \text{otherwise}
                    \end{cases}
    \end{aligned}
\end{equation}
For each training instance $i$, at any given epoch $t$, the DDS mask $\Gamma_M$ selects, instead of $M$ features, all $F$ features, with a probability lower than $\epsilon$. By default, $\epsilon$ is set to 0.1.

% Besides that, the methodology training remains consistent with the main task, requiring only the incorporation of the DDS architecture into the original problem framework. By seamlessly integrating DDS into our existing methodology, a streamlined transition that enhances the efficiency and effectiveness of this approach is ensured. This strategic augmentation empowers us to tackle the complexities of the main task with precision and agility, leveraging the inherent adaptability and robustness of DDS.

\section{Experiments}
\label{sec:results}

The experimental section of this paper explores two main tasks. First, the DDS model is attached to a state-of-the-art clustering technique, to show how its inclusion can dramatically increase its performance; Second, DDS is used to learn a representation to be used by an agent in a World Model. Results show that DDS increases both the agent performance and the visual quality of the reconstructions. All algorithms, scripts, and results are accessible via GitHub.\footnote{URL will be included upon acceptance.} All experiments were conducted on workstation with a NVIDIA GeForce RTX 3090Ti GPU (24GB memory), a 12th Gen Intel(R) Core(TM) i7-12700K processor (20 cores), and 64GB RAM.

% Aquí ía o de data compression. Agora metemos o World Model como 4.2 despois de clustering
% \input{old4_1-data-compression}

\subsection{Clustering}
\label{subsec:clustering}

First, we explored the use of DDS in a clustering scenario. Our model was integrated into a state-of-the-art architecture without any calibration or hyperparameter tuning.

Over recent years, contrastive learning has boosted the quality of unsupervised image clustering. Techniques like Contrastive Clustering (CC) \cite{li2021contrastive} or its upgrade, the Twin Contrastive Clustering (TCL) \cite{li2022twin} achieved remarkable results nearing those obtained by supervised techniques. However, some of these works perform image resizing, changing the initial \texttt{CIFAR-10} image size $32 \times 32$ to a much bigger one ($224 \times 224$ for TCL, for instance). Therefore, it makes no sense to perform a dynamic feature selection over an artificially enlarged image.

To make a fair experiment, DDS should be attached to a clustering model that does not require image resizing to achieve state-of-the-art results. Over these solutions, ProPos \cite{huang2022learning} stands out from the rest. Therefore, the DDS module $g(\cdot; \mathbf{\Theta}_S)$ was attached to the ProPos model $f(\cdot; \mathbf{\Theta}_U)$, without any hyperparameter tuning. The same training procedure presented in \cite{huang2022learning} was applied, although the number of epochs was increased to $2000$ for \texttt{Tiny-Imagenet} and $4000$ in the other datasets. This was done because ProPos’ training plateaus long before it reaches the tested number of epochs, and its results do not change significantly for longer training. In contrast, our approach continues to improve its performance over time. The U-Net model (with $C = 16$) was used as the DDS architecture, and two values of $M$ were tested: 10\% and 25\% of the total image pixels. As in \cite{huang2022learning}, a ResNet-18 architecture was used as the backbone for the Tiny-Imagenet dataset, whereas a ResNet-34 was used for the others. 

As baselines, we used  SCAN \cite{van2020scan}, NMM \cite{dang2021nearest}, GCC \cite{zhong2021graph}, TCC \cite{shen2021you}, SimSiam \cite{chen2020simple}, BYOL \cite{grill2020bootstrap}, IDFD \cite{tao2020clustering} and PCL \cite{li2020prototypical} in addition to the aforementioned CC, TCL, and ProPos. Table \ref{tab:clustering} shows the clustering results obtained over four different datasets: \texttt{CIFAR-10}, \texttt{CIFAR-20} \cite{krizhevsky2009learning}, \texttt{ImageNet-10} \cite{ILSVRC15} and \texttt{ImageNet-Dogs} \cite{KhoslaYaoJayadevaprakashFeiFei_FGVC2011}. The results show that DDS significantly reduced the number of input features without hindering the performance on small images. Furthermore, for larger images(\texttt{ImageNet-10} and \texttt{ImageNet-Dogs} models use $224 \times 224$ images), DDS achieved state-of-the-art results when selecting, per sample, only one-quarter of their features. For the \texttt{Tiny-ImageNet} dataset (using $64 \times 64$ images) DDS reaches new state-of-the-art results and surpasses ProPos in the three metrics by a margin greater than $15\%$. Moreover, doubling the number of epochs increases this improvement to an impressive $80\%$ in the ARI score (54.6 NMI, 40.3 ACC and 26.3 ARI), while it does not affect the performance of the original ProPos. These results suggest that removing unnecessary information, such as background details, can enormously help the clustering algorithm.

\begin{table*}[t]
\caption{Clustering results on four different datasets. The best and second-best results are shown in bold and underlined, respectively. The DDS+ProPos model maintains similar results to those obtained by the original ProPos, even when using much fewer selected features.}
\label{tab:clustering}
\begin{center}
\resizebox{.99\linewidth}{!}{
\begin{tabular}{|l|ccc|ccc|ccc|ccc|ccc|}
\hline
\textbf{Method} & \multicolumn{3}{c|}{\texttt{CIFAR-10}} & \multicolumn{3}{c|}{\texttt{CIFAR-20}} & \multicolumn{3}{c|}{\texttt{ImageNet-10}} & \multicolumn{3}{c|}{\texttt{ImageNet-Dogs}} & \multicolumn{3}{c|}{\texttt{Tiny-Imagenet}} \\ \hline
 & \textbf{NMI} & \textbf{ACC} & \textbf{ARI} & \textbf{NMI} & \textbf{ACC} & \textbf{ARI} & \textbf{NMI} & \textbf{ACC} & \textbf{ARI} & \textbf{NMI} & \textbf{ACC} & \textbf{ARI} & \textbf{NMI} & \textbf{ACC} & \textbf{ARI} \\ \hline
%IIC \cite{ji2019invariant} & 51.3 & 61.7 & 41.1 & - & 25.7 & - & - & - & - & - & - & - \\
%DCCM \cite{wu2019deep} & 49.6 & 62.3 & 40.8 & 28.5 & 32.7 & 17.3 & 60.8 & 71.0 & 55.5 & 32.1 & 38.3 & 18.2 \\
%PICA \cite{huang2020deep} & 56.1 & 64.5 & 46.7 & 29.6 & 32.2 & 15.9 & 78.2 & 85.0 & 73.3 & 33.6 & 32.4 & 18.2 \\
SCAN & 79.7 & 88.3 & 77.2 & 48.6 & 50.7 & 33.3 & - & - & - & - & - & - & - & - & - \\
NMM & 74.8 & 84.3 & 70.9 & 48.4 & 47.7 & 31.6 & - & - & - & - & - & - & - & - & -\\
CC & 70.5 & 79.0 & 63.7 & 43.1 & 42.9 & 26.6 & 85.9 & 89.3 & 82.2 & 44.5 & 42.9 & 27.4 & 34.0 & 14.0 & 7.1\\
%MiCE \cite{tsai2020mice} & 73.7 & 83.5 & 69.8 & 43.6 & 44.0 & 28.0 & - & - & - & 42.3 & 43.9 & 28.6 \\
GCC & 76.4 & 85.6 & 72.8 & 47.2 & 47.2 & 30.5 & 84.2 & 90.1 & 82.2 & 49.0 & 52.6 & 36.2 & 34.7 & 13.8 & 7.5\\
TCL & 81.9 & 88.7 & 78.0 & 57.9 & 53.1 & 35.7 & 87.5 & 89.5 & 83.7 & 62.3 & 64.4 & 51.6 & - & - & -\\
TCC & 79.0 & 90.6 & 73.3 & 47.9 & 49.1 & 31.2 & 84.8 & 89.7 & 82.5 & 55.4 & 59.5 & 41.7 & - & - & -\\
%MoCo \cite{he2020momentum} & 66.9 & 77.6 & 60.8 & 39.0 & 39.7 & 24.2 & - & - & - & 34.7 & 33.8 & 19.7 \\
SimSiam & 78.6 & 85.6 & 73.6 & 52.2 & 48.5 & 32.7 & 83.1 & 92.1 & 83.3 & 58.3 & 67.4 & 50.1 & 35.1 & 20.3 & 9.4\\
BYOL & 81.7 & 89.4 & 79.0 & 55.9 & 56.9 & 39.3 & 86.6 & 93.9 & 87.2 & 63.5 & 69.4 & 54.8 & 36.5 & 19.9 & 10.0\\
IDFD & 71.1 & 81.5 & 66.3 & 42.6 & 42.5 & 26.4 & 89.8 & 95.4 & 90.1 & 54.6 & 59.1 & 41.3 & - & - & -\\
PCL & 80.2 & 87.4 & 76.6 & 52.8 & 52.6 & 36.3 & 84.1 & 90.7 & 82.2 & 44.0 & 41.2 & 29.9 & 35.0 & 14.0 & 7.1\\
ProPos & \textbf{88.6} & \textbf{94.3} & \textbf{88.4} & \underline{60.6} & \textbf{61.4} & \underline{45.1} & 89.6 & 95.6 & 90.6 & 69.2 & 74.5 & 62.7 & 40.5 & 25.6 & 14.3\\ \hline
DDS(10\%)+ProPos & \multirow{1}{*}{80.2} & \multirow{1}{*}{86.3} & \multirow{1}{*}{75.7} & \multirow{1}{*}{54.4} & \multirow{1}{*}{50.5} & \multirow{1}{*}{37.2} & \multirow{1}{*}{\textbf{91.8}} & \multirow{1}{*}{\textbf{96.7}} & \multirow{1}{*}{\textbf{92.8}} & \multirow{1}{*}{\underline{74.4}} & \multirow{1}{*}{\underline{76.0}} & \multirow{1}{*}{\underline{66.5}} & \underline{43.8} & \underline{27.1} & \underline{16.0}\\
%& & & & & & & & & & & & \\
DDS(25\%)+ProPos & \multirow{1}{*}{\underline{87.6}} & \multirow{1}{*}{\underline{93.9}} & \multirow{1}{*}{\underline{87.2}} & \multirow{1}{*}{\textbf{62.2}} & \multirow{1}{*}{\underline{58.4}} & \multirow{1}{*}{\textbf{46.6}} & \multirow{1}{*}{\underline{90.8}} & \multirow{1}{*}{\underline{96.2}} & \multirow{1}{*}{\underline{91.7}} & \multirow{1}{*}{\textbf{75.9}} & \multirow{1}{*}{\textbf{78.6}} & \multirow{1}{*}{\textbf{69.5}} & \textbf{47.0} & \textbf{30.5} & \textbf{18.9}\\
\hline
\end{tabular}
}
\end{center}
\end{table*}

% \begin{table}[ht]
% \centering
% \caption{Clustering Results (\%) on \texttt{Tiny-ImageNet}. The best and second-best results are shown in bold and underlined, respectively. The Long version was trained by doubling the number of epochs.}
% \label{tab:tiny-imagenet}
% \begin{tabular}{|l|c|c|c|}
% \hline
% \textbf{Method} & \textbf{NMI} & \textbf{ACC} & \textbf{ARI} \\ \hline
% %DCCM \cite{wu2019deep} & 22.4 & 10.8 & 3.8 \\
% %PICA \cite{huang2020deep} & 27.7 & 9.8 & 4.0 \\
% CC \cite{li2021contrastive} & 34.0 & 14.0 & 7.1 \\
% GCC \cite{zhong2021graph} & 34.7 & 13.8 & 7.5 \\
% % MoCo \cite{he2020momentum} & 34.2 & 16.0 & 8.0 \\
% PCL \cite{li2020prototypical} & 35.0 & 15.9 & 8.7 \\
% SimSiam \cite{chen2020simple} & 35.1 & 20.3 & 9.4 \\
% BYOL \cite{grill2020bootstrap} & 36.5 & 19.9 & 10.0 \\
% ProPos \cite{huang2022learning} & 40.5 & 25.6 & 14.3 \\ \hline
% DDS(10\%) + ProPos & 43.8 & 27.1 & 16.0\\
% DDS(25\%) + ProPos & \underline{47.0} & \underline{30.5} & \underline{18.9}\\
% DDS(25\%) + ProPos (Long) & \textbf{54.6} & \textbf{40.3} & \textbf{26.3}\\
% \hline
% \end{tabular}
% \end{table}

\subsection{World models}
\label{world-models-introduction}
The reinforcement learning problem of constructing agents that learn to interact with a dynamic environment has recently been tackled using world models. These models serve as generative frameworks for simulating environments internally and enable agents to predict and act based on imagined scenarios rather than relying on direct interactions with their surroundings. The usual architecture of these agents~\cite{DBLP:journals/corr/abs-1803-10122} relies on a separately trained vision model to construct a compact and structured latent representation of the environment which the agent then uses to determine the actions to be taken. A complete description is provided in Appendix~\ref{appendix:world-models}.

% These latent representations are mainly used to predict future states of the environment. However, they can also be leveraged to generate visualizations of these new states. Lately, this has gained a lot of interest in the field of game generation, using world models as real-time game engines. This generative capability enables the creation of unique procedural worlds based on the Player's actions. 

The vision model was originally solved by using a Variational Autoencoder (VAE) \cite{kingma2013auto}, which is known to produce reconstructions that are blurry \cite{tomczak2018vae}.To provide a more challenging and recent comparison, we also evaluate against a Masked Autoencoder (MAE)~\cite{he2021maskedautoencodersscalablevision} based vision model.

We propose to maintain the architecture in~\cite{DBLP:journals/corr/abs-1803-10122} but enhance the Vision model (V), which obtains the latent representations of the environment, with our DDS model. The goal is to improve the efficacy of the obtained representations both in terms of agent performance and the visual quality of the generated images. The DDS architecture introduced in Fig.~\ref{fig:architecture} is adapted by introducing a VAE to learn latent interpretations for the agent to use in the process of learning to reconstruct the mask selected by DDS. Changing the goal of the VAE to reconstructing the masked input ($\textbf{mask}=g(\mathbf{X}) \circ \mathbf{X})$) instead of the original input ($\textbf{X}$) simplifies the task of the VAE by removing noise, which allows it to obtain better representations. Fig. \ref{fig:model_architecture} describes the new architecture.

\begin{figure*}[t]
\centering
\includegraphics[width=1.0\textwidth]{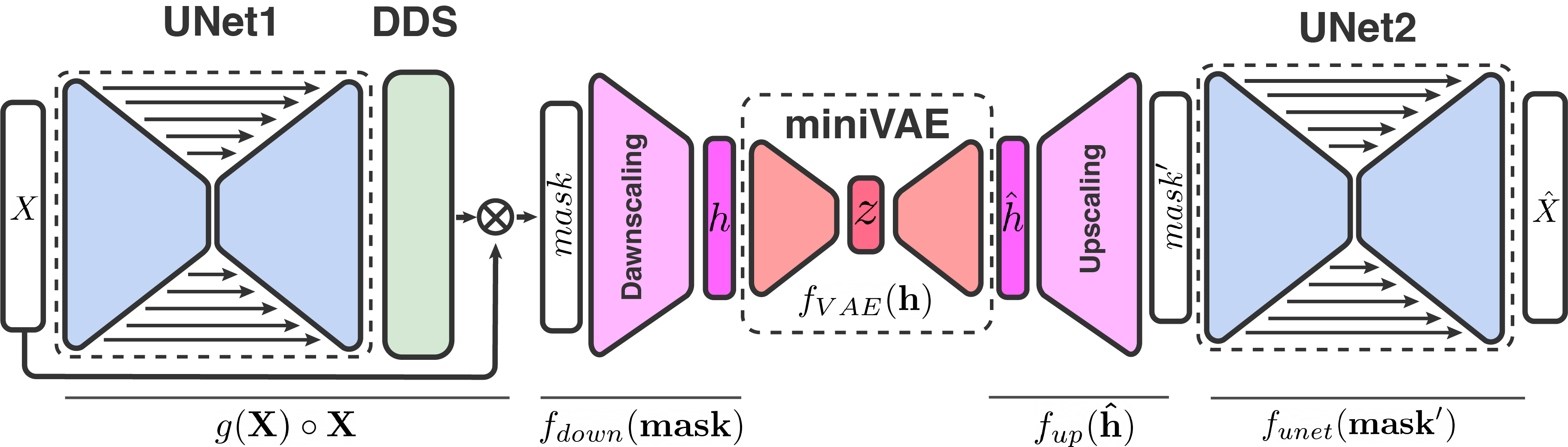}
\caption{Adaptation of the DDS architecture to the World Model problem. The previously presented DDS architecture (green and blue) is augmented to yield a structured latent space with the addition of a Variational Autoencoder (red) that aims to reconstruct the masked inputs (i.e. the relevant features of the input image). The training procedure is divided in two steps: (1) the DDS is trained to learn to select the relevant features of each image without the \textit{VAE} section (i.e. $\mathbf{h}=\mathbf{\hat{h}}$), and then (2) the \textit{VAE} is trained to compute the $\mathbf{\hat{h}}$ reconstruction of $\mathbf{h}$.}
\label{fig:model_architecture}
\end{figure*}

The training procedure has two steps: First, everything but the VAE is trained. To train as much of the network as possible in this first stage, the VAE is split into three functions $\textbf{mask'}=f_{up}(f_{VAE}(f_{down}(\textbf{mask})))$ and $f_{VAE}$ is removed in this training step. The function to be learned is, therefore, $f_{unet}(f_{up}(f_{down}(g(\mathbf{X}) \circ \mathbf{X})))$. The main idea is to learn to select the relevant features, as well as a lower resolution representation $\mathbf{h}=f_{down}(g(\mathbf{X}) \circ \mathbf{X})$ representing key points of the image structure and texture. The primary optimization objective is to minimize a pixel-level reconstruction loss, typically implemented as the mean squared error (MSE) between the input image \(\mathbf{x}\) and its reconstruction \(\mathbf{\hat{x}}\). The second step begins once both U-Nets ($g$ and $f_{unet}$) and the downscaling/upscaling modules ($f_{down}$ and $f_{up}$) have been fully trained and frozen. We train the \emph{VAE} to reconstruct $\textbf{mask}=(g(\mathbf{X}) \circ \mathbf{X})$. In this process, it learns a more compact latent vector \(\mathbf{z} \in \mathbb{R}^{d}\). Unlike standard VAEs \cite{kingma2022autoencodingvariationalbayes} that directly compress high-resolution inputs (often requiring large, deep networks), our VAE leverages the frozen $f_{down}$ and $f_{up}$, and trains only $f_{VAE}$. With this approach, the reconstruction loss approaches a \emph{perceptual loss} \cite{hou2017deep} without requiring supervised pre-trained models: it computes the MSE between the output of each layer of $f_{up}$ when applied to either the latent representation $\mathbf{h}$ and its reconstruction $\mathbf{\hat{h}}$. 

Finally, a known drawback of VAEs is that strict adherence to the \(\mathcal{N}(0,1)\) prior can lead to dimensional collapse: a few latent dimensions carry most of the information, while others remain dormant. To address this, we adopt the \emph{free bits} method \cite{DBLP:journals/corr/KingmaSW16}, which modifies the VAE loss to allow each dimension a small ``budget” of KL divergence.

Formally speaking, the VAE loss function is defined as:

\begin{equation}
\label{eq:mini_vae_loss_compact}
\begin{aligned}
\mathcal{L}_{\text{VAE}} 
&= \mathcal{L}_{\text{perceptual}} + \alpha \max\big(\lambda, D_{\mathrm{KL}}(q(\mathbf{z}|\mathbf{h}) \parallel p(\mathbf{z}))\big), \\
\mathcal{L}_{\text{perceptual}} 
&= \dfrac{1}{|f_{\mathrm{up}}|} \sum_{k=1}^{|f_{\mathrm{up}}|} \| f_{\mathrm{up}}(\hat{\mathbf{h}})_k - f_{\mathrm{up}}(\mathbf{h})_k \|^2,
\end{aligned}
\end{equation}
$f_{\mathrm{up}}(\mathbf{X})_k$ is the output of the $k$-th layer of the upscaling module, and $\alpha=1$ and $\lambda = 0.02$ are some regularization terms for the \emph{free bits penalization.}

In our study, we followed the same procedure for dataset collection, model training, and evaluation, ensuring that our experiments remained consistent with the original \cite{DBLP:journals/corr/abs-1803-10122} methodology, ensuring that any observed differences in performance and image fidelity could be attributed solely to the changes in the latent space construction method. We use the same latent dimension size (\(\textbf{z}\in \mathbb{R}^{32}\)) for the original VAE, our DDS+VAE, and the MAE+VAE baseline. The number of parameters for the vision models is: 4,348,547 for the original VAE \cite{DBLP:journals/corr/abs-1803-10122}, 6,686,464 for the MAE+VAE baseline, and notably, 4,039,089 for our DDS+VAE approach. This highlights that our method can achieve substantial gains in reconstruction fidelity and downstream RL performance (as shown later) with a more parameter-efficient model compared to both the original VAE and the more complex MAE+VAE.

\textbf{Image quality experiment:} \quad
To assess the quality of the reconstructed images, we measured the reconstruction error of the different models. The results shown in Table \ref{tab:mse_world_model} highlight that our DDS+VAE approach yields a noticeably lower reconstruction error than both the baseline VAE and the MAE+VAE when reconstructing. Specifically, DDS+VAE (with 4\% or 8\% pixel selection) achieves significantly lower MSE than VAE (0.00165) and MAE+VAE (0.00220). This indicates that DDS successfully pinpoints crucial image regions — capturing objects, textures, and boundaries — thus enabling the model to more accurately reconstruct the environment. For a qualitative illustration of the intermediate representations, input reconstructions, and generated dream sequences for both \texttt{CarRacing-v3} \citep{gymnasium2025} and \texttt{SuperMarioBros-v0} \citep{gym-super-mario-bros} environments, please refer to Appendix~\ref{appendix:visualization}.

\begin{table*}[ht]
\centering
\caption{Reconstruction MSE over 20.000 inference dataset from the \texttt{CarRacing-v3} environment using different $M$ values. DDS configuration is the same as provided for Fig. \ref{fig:model_architecture}, except for the $M$ values included in each row.}
\label{tab:mse_world_model}
\vspace{0.4em} 
\resizebox{1.0\columnwidth}{!}{
\begin{tabular}{|p{6cm}|c|c|c|c|c|c|}
\hline
\textbf{Vision model architecture} & \textbf{2\%} & \textbf{4\%} & \textbf{8\%} & \textbf{16\%} & \textbf{32\%} & \textbf{100\%} \\ \hline
Baseline (VAE) & - & - & - & - & - &  0.00165\\
Baseline (MAE+VAE) & - & - & - & - & - &  0.00220\\

\hline
% DDS w/o Residual Links & 0.00328 & 0.00164 & 0.00188 & 0.00228 & 0.00352 \\
% \hline
% DDS Hard Sigmoid ($\zeta = 1.1$, $\gamma = -0.1$) & 0.00060 & 0.00095 & 0.00089 & 0.00099 & 0.00119 \\
% \hline
% DDS classic Binary Concrete ($\kappa = 1.$) & 0.00134 & 0.00039 & 0.00042 & 0.00047 & 0.00051\\
% \hline
% DDS w/o Binary Concrete ($\kappa = 0.$) & 0.00157 & 0.00099 & 0.00224 & 0.00063 & 0.00083\\
% \hline
% DDS w/o Dynamic $M$ ($\epsilon = 0$)  & 0.00086 & 0.00094 & 0.00148 & 0.00090 & 0.00101\\
% \hline
Proposed (DDS + VAE) & 0.00134 & \textbf{0.00039} & 0.00042 & 0.00047 & 0.00051 & - \\
\hline
\end{tabular}}
\end{table*}

After training the Memory model using our new latent representation \(\mathbf{z}\), we can generate \emph{dream sequences} that simulate future states of the environment without direct interaction. 
%Specifically, we begin by encoding a \emph{real} image \(\mathbf{x}^{(0)}\) from the environment into its corresponding latent vector \(\mathbf{z}_t^{(0)}\) using the trained encoder pathway. Subsequently, rather than relying on additional real observations, we recursively predict future latent states \(\mathbf{z}_{t+1}^{(0)}\) using the Memory model outputs from previous time-steps as described in \cite{DBLP:journals/corr/abs-1803-10122}. 
Figure~\ref{fig:dream_combined} shows the contrast between the baseline's results and our method's. The images \emph{dreamed} by the baseline model are often blurry and show artifacts that make them visually dissimilar from the real environment. Furthermore, although the predicted dynamics remain loosely consistent with real motions, these generated states sometimes diverge significantly from the appearance and dynamics of the actual environment (particularly when the agent executes sharp turns or accelerates rapidly). In contrast, our approach yields dream sequences that retain more detail and exhibit smoother transitions between frames. 
% While these simulations are still an approximation and cannot flawlessly replicate real physics, they better capture the look and feel of the environment. For instance, road boundaries, vehicle positioning, and background textures evolve more in line with what one would observe in the real environment, especially when the dream is driven by actions from a human player. 
This enhanced internal simulation capability can facilitate more effective training of downstream policies, as shown in the next experiment.

We evaluated the proposed DDS+VAE model's ability to generate dream sequences compared to the VAE and our MAE+VAE baselines. We generated and analyzed four sets of sequences for each environment (\texttt{CarRacing-v3} and \texttt{SuperMarioBros-v0}), each containing 100 sequences with 1,000 frames. The first set consists of real environment sequences, the subsequent sets comprise dream sequences generated using the VAE, MAE+VAE, and our proposed DDS+VAE model, respectively.

\begin{table}[htbp]
    \centering
    \caption{FID and FVD scores for different models over 100 sequences (lower is better). The DDS+VAE results are for $M=4\%$ feature selection.}
    \label{tab:quality_coherence}
    \resizebox{0.6\columnwidth}{!}{
    \begin{tabular}{llccc}
        \toprule
        \textbf{Dataset} & \textbf{Metric} & \textbf{VAE} & \textbf{MAE+VAE} & \textbf{DDS+VAE} \\
        \midrule
        \multirow{2}{*}{CarRacing-v3} & FID & 59.46 & 54.84 & \textbf{25.35} \\
                                      & FVD & 239 & 312 & \textbf{176} \\
        \midrule
        \multirow{2}{*}{SuperMarioBros-v0} & FID & 64.06 & \textbf{60.21} & 61.08 \\
                                           & FVD & 412 & 465 & \textbf{338} \\
        \bottomrule
    \end{tabular}}
\end{table}

To assess the quality of the generated sequences, we used two standard metrics: Fréchet Inception Distance (FID)\cite{DBLP:journals/corr/HeuselRUNKH17}, which measures image distribution similarity using features from a pre-trained Inception V3 model, and Fréchet Video Distance (FVD)\cite{DBLP:journals/corr/abs-1812-01717}, which extends FID to evaluate temporal coherence using a pre-trained I3D model. Lower FID and FVD scores indicate higher visual fidelity and better sequence consistency, respectively. As shown in Table~\ref{tab:quality_coherence}, for the \texttt{CarRacing-v3} environment, the DDS+VAE model achieved significantly lower FID and FVD scores compared to both the VAE and MAE+VAE baselines, demonstrating superior image reconstruction quality and enhanced temporal coherence. For the \texttt{SuperMarioBros-v0} environment, our DDS+VAE model attained the best FVD score, indicating superior temporal coherence. While MAE+VAE achieved a slightly lower FID in this environment, our DDS+VAE remained highly competitive and presents a better overall profile, especially when considering its parameter efficiency and strong FVD. These results demonstrate the robustness and effectiveness of our DDS+VAE approach across different visual domains and against strong baselines.

\begin{figure*}[htpb]
\centering
    \begin{subfigure}
        \centering
        \includegraphics[width=\textwidth]{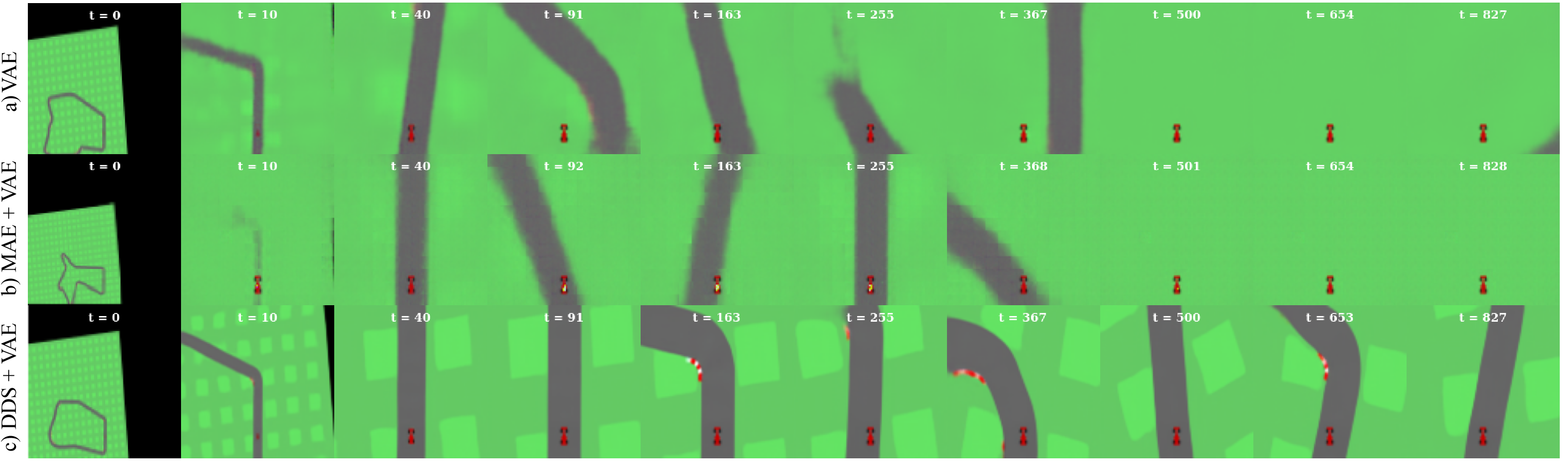}
    \end{subfigure}
    \caption{Comparison of dream sequence generation by the World Model. a) Original Vision model with a VAE architecture. b) Proposed (VAE+DDS with M=4\%) as Vision model.}
    \label{fig:dream_combined}
\end{figure*}

% Formally, at each step \(t\), the MDN-RNN updates its hidden state \(\mathbf{h}_t\) according to:
% \[
% \mathbf{h}_{t+1} \sim \mathcal{P}(\mathbf{h}_{t+1} \mid \mathbf{a}_t, \mathbf{z}_t, \mathbf{h}_t),
% \]
% where \(\mathbf{a}_t\) is the agent’s (or human player’s) action at time \(t\). The MDN-RNN then produces a probability distribution for \(\mathbf{z}_{t+1}\) conditioned on \(\mathbf{z}_t\), \(\mathbf{a}_t\), and \(\mathbf{h}_t\). Sampling from this distribution yields the next latent state \(\mathbf{z}_{t+1}\). Finally, \(\mathbf{z}_{t+1}\) is passed through the \emph{decoder} stage of our model (Upscaling CNN + U-Net) to generate the corresponding “dreamed” image \(\mathbf{\hat{x}}_{t+1}\). This output then replaces any real observation \(\mathbf{x}_{t+1}\) that would otherwise come from the environment, allowing us to “play” with the internal simulation of the world model.

% \newpage
\textbf{Agent performance:} \quad
Finally, we examine the agent's performance when deciding its policy using the latent representations obtained by the original and proposed Vision models (DDS+VAE with M=4\%) . The evaluation, conducted over 100 episodes for each controller, highlights the the new model's improvements. The original 2018 architecture achieved an average reward of 734.96 ± 162.75, whereas the new model reached 818.58 ± 147.05, confirming its enhanced performance.

% Figure~\ref{img:controller_performance} shows that the latent representations lead to superior performance, although the agent learns slightly slower. 

% \begin{figure}[ht]
% \vskip 0.2in
% \begin{center}
% \centerline{\includegraphics[width=.6\textwidth]{images/cma_es.png}}
% \caption{Average reward of the best-performing individual in the population. Two curves are shown: one representing the evolution of the original 2018 world model architecture and the other showing the improvement brought by the new vision.
% % The updated model exhibits a slower initial increase in reward, as the training begins with a more complex representation. However, as training progresses, it surpasses the original model's performance, demonstrating a slight but consistent improvement in rewards. 
% % The final evaluation, conducted over 100 episodes for each controller, highlights the advancements of the new model. The original 2018 architecture achieved an average reward of 734.96 ± 162.75, whereas the new model reached 818.58 ± 147.05, confirming its enhanced performance.
% }
% \label{img:controller_performance}
% \end{center}
% \vskip -0.2in
% \end{figure}

\subsection{Ablation Study}
\label{sec:ablation}
Multiple configurations were tested to determine both the correct model hyperparameters and the limitations of the architecture. We measured performance changes for different configurations using the \texttt{CIFAR-10} dataset.

\begin{table*}[ht]
\centering
\caption{Reconstruction MSE over \texttt{CIFAR-10} retaining different $M$ features.}
\label{tab:ablation}
\resizebox{0.75\columnwidth}{!}{
\begin{tabular}{|p{6cm}|c|c|c|c|c|}
\hline
\textbf{M} & \textbf{64} & \textbf{128} & \textbf{256} & \textbf{512} & \textbf{1024} \\ \hline
Naive AE & 0.01820 & 0.01264 & 0.00806 & 0.00526 & 0.00449 \\
\hline
DDS w/o Residual Links & 0.03129 & 0.01804 & 0.01302 & 0.00844 & 0.00817 \\
\hline
DDS Hard Sigmoid ($\zeta = 1.1$, $\gamma = -0.1$) & 0.01838 & 0.00861 & 0.00397 & 0.00126 & 0.00030 \\
\hline
DDS classic Binary Concrete ($\kappa = 1.$) & 0.03759 & 0.01771 & 0.01415 & 0.00320 & 0.00064\\
\hline
DDS w/o Binary Concrete ($\kappa = 0.$) & 0.01641 & 0.01145 & 0.00484 & 0.00148 & 0.00024\\
\hline
DDS w/o Dynamic $M$ ($\epsilon = 0$)  & 0.01609 & 0.00778 & 0.00385 & 0.00088 & 0.00017 \\
\hline
\rowcolor[gray]{.9} DDS only for training & 0.17394 & 0.18085 & 0.18379 & 0.22682 & 0.19937 \\
\hline
DDS & 0.01636 & 0.00945 & 0.00469 & 0.00119 & 0.00025\\
\hline
\end{tabular}}
\end{table*}

Table \ref{tab:ablation} shows the results of different MSE reconstruction configurations when using an autoencoder-like configuration (see Fig \ref{fig:architecture}). The first and most important insight from this experiment is that the most advanced architecture of the model plays a pivotal role in the accuracy of the solution. DFS outperforms the naive AutoEncoder by preserving spatial information while maintaining a compact representation with fewer variables. This enables the use of advanced networks with residual links, whereas traditional feature selection methods require a latent representation comparable in size to the original input to leverage such architectures. Another insight is that using the hard sigmoid configuration provided in \cite{balin2019concrete} drops the performance by a slight margin, especially when selecting low $M$ values.

This phenomenon is caused by the zero gradient obtained when the feature importance is cropped (values higher than 1 or lower than 0) in early training stages, causing a limitation in the ability of the model to adapt. This validates the change proposed in Section \ref{sec:DDS_implementation}.

Lastly, the use of the Binary Concrete distribution was tested. Using the classic configuration ($\kappa = 1$) resulted in a drop in performance. Since the reconstruction task is a regression problem, it was found that high variations in the output of the DDS model result in the inability of the reconstruction model ($\Theta_{U}$) to achieve good generalization. However, when setting $\kappa = 0$ the results are comparable with the ones obtained by our default configuration, suggesting that, for this reconstruction task, the Binary Concrete distribution is useless. A similar effect occurs when removing the dynamic $M$ variation provided in Eq. \ref{eq:dynamic_M} ($\epsilon = 0$). In this case, the results suggest that using it may be counterproductive.

\begin{wrapfigure}[13]{r}{0.35\textwidth}
    \vspace{-1.0em} 
    \centering
    \includegraphics[width=0.35\textwidth]{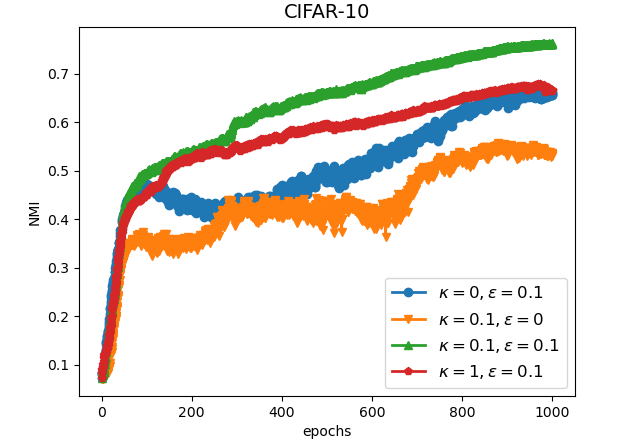}
    \caption{
        DDS(10\%) + ProPos clustering NMI over CIFAR-10, using a ResNet-18 as backbone.
    }
    \label{fig:ablation}
\end{wrapfigure}

In contrast, Fig. \ref{fig:ablation} shows the NMI clustering results when training DDS(10\%) + ProPos over $1000$ epochs, using a ResNet-18 as the backbone. In this case, the default configuration achieves the best results. The result when removing the dynamic $M$ variation is of special interest, since not only obtain the worst results, but also provide a more unstable output. The same problem arises when no Binary Concrete distribution is included in the training procedure ($\kappa = 0$).

Since the aim of this work to provide a single set of useful hyperparameters, no matter which type of problem needs to be solved, the default configuration $\kappa = 0.1$, $\epsilon = 0.1$, $\zeta = 1$, $\gamma = 0$ obtains good solutions in both experiments. However, the results can be improved if it is carefully tuned for a specific task.
%It is specially interesting to note that, even when only selecting the 10\% of the input features, the score obtained is higher than state-of-the-art techniques of 4 years ago. These results suggest that DDS can be attached to any state-of-the-art model with almost no clustering performance loss.

\section{Conclusion}
\label{sec:conclusion}

This paper presents a general recipe for Dynamic Feature Selection in unsupervised scenarios. The presented DDS module can be attached to the input of any architecture tackling an unsupervised task. The module consists of an autoencoder-like architecture that outputs the selection of, at most, $M$ relevant features with their respective score, with $M$ being a fixed parameter tuned by the operator. Our experiments show that DDS can perform data compression with better results than the alternatives, even when accounting that extra memory is needed for saving the feature selection indexes. This improvement is caused by two factors: first, the information provided by the selected features is extremely discriminative; and second, DDS allows the use of more complex downstream architectures since the input data structure is always preserved. Finally, we show that the DDS architecture can be attached to two different architectures, tackling very different problems, improving the performance in both cases.

% \subsection{Limitations}

\textbf{Limitations:} When using the DDS module on an existing architecture, the training procedure must be adapted. The number of epochs often needs to be doubled (compared to training the same architecture without the DDS module).

It is worth noting that the output of the DDS architecture is not forced to be binary. In fact, preliminary studies show that the feature importance score rarely reaches the perfect score of 1. This reduces explainability, as the stored compressed data is modified from the original. However, this can be solved by storing the input data and their importance scores separately, although this almost doubles the memory requirements. If a small memory footprint is a requirement, the DDS output can be forced to be binary by introducing more restrictions into the model, although initial tests show significant degradation in performance. %, only a small fraction of techniques were used to force this extra restriction.

%It is also interesting to remark that, in the clustering section, some experiments were performed using high $M$ values ($0.5 F$ and $F$) over the CIFAR-10 dataset. As good results were obtained when using the ResNet-18 architecture as backbone, erratic results were obtained when switching to the ResNet-34 used in the experiments provided in Table \ref{tab:clustering}. However, this problem was reduced when some $l_0$ regularization was introduced into the DDS output, although the results are inferior to the ones provided by DDS(25\%).

%\subsection{Future Work}

\textbf{Future Work:} We plan to take advantage of the ability of the DDS architecture to preserve the input data structure. Novel contrastive learning loss functions can be derived from this idea, as the selected pixels of an image should be similar no matter how many geometrical operations are applied to perform data augmentation over them. Finally, it would also be interesting to extend this architecture to the supervised scenario, as in previous DFS algorithms.

\section*{Impact Statement}
\label{sec:impact}
Our Dynamic Feature Selection framework enhances the interpretability of machine learning models by identifying the most relevant components of each input sample. This capability can deepen trust in data-driven solutions, as decision pathways become more transparent and directly traceable to their most essential inputs. In addition, the method’s flexible design facilitates integration with diverse unsupervised tasks, from clustering to generative modeling, broadening its potential impact in both academic research and real-world deployments.

\bibliographystyle{unsrtnat}
\bibliography{references}

\clearpage
\appendix

\section{Intuition Behind The Idea}
\label{sec:proof}

In this section we will discuss the intuition behind the use of the DFS, and why removing irrelevant features helps to improve an unsupervised model's performance when attached to it.

\subsection{Effect of Small and Irrelevant Inputs in Neural Networks}

Let \( f(\cdot): \mathbb{R}^n \to \mathbb{R}^m \) be a neural network function approximating some unsupervised task. Assume the input is \( x = (x_1, x_2, \dots, x_n) \in \mathbb{R}^n \), and we partition it as:

\begin{equation}
    x = (x_{\text{rel}}, x_{\text{irr}})
\end{equation}

where:

\begin{itemize}
    \item \( x_{\text{rel}} \in \mathbb{R}^r \): Relevant features
    \item \( x_{\text{irr}} \in \mathbb{R}^{n - r} \): Irrelevant (noisy or uninformative) features
\end{itemize}

Let $\epsilon^j = (0, \dots, 0, \underbrace{\epsilon}_j, 0, \ldots, 0 ) \in \mathbb{R}^n$ a one-hot vector where $\epsilon^j_j \approx 0$ and $x^{\text{rel}} = (x_{\text{rel}}, 0)$.

\vspace{.5cm}

\begin{lemma}
    \label{lemma:1_1}
    In deep networks, could $\exists ~i \in [n-r, n] \quad | \quad \| f(x^{\text{rel}} + \epsilon^i) - f(x^{\text{rel}}) \| \gg 0$.
\end{lemma}
\begin{proof}
    For the sake of simplicity, lets assume a neural network with ReLU activations. Even though ReLU is piecewise linear and zero for negative inputs, it can still allow small irrelevant inputs to influence the output significantly through deep weight paths.

    Assume a deep ReLU network:
    \begin{equation}
        f(x) = W_L \cdot \text{ReLU}(W_{L-1} \cdot \dots \cdot \text{ReLU}(W_1 x))
    \end{equation}

    Let \( x^{\text{rel}}_i + \epsilon^i_i \approx 0 \) and suppose that the ReLU activation is in the linear (active) regime for all layers. If the weights \( w_{li} \gg w_{lj}, ~\forall j \in [1, r] \), then:
    \begin{equation}
        f(x^{\text{rel}} + \epsilon^i) \propto \epsilon \cdot \prod_{l=1}^L w_{li} ~\gg~ x_j \cdot \prod_{l=1}^L w_{lj}, \quad \forall j \in [1, r]
    \end{equation}

    This means that even though \( x_i \) is very small, its effect can grow exponentially with depth if each layer multiplies the value by a large weight, shadowing the importance of the relevant features in the model's output.

\end{proof}

\begin{lemma}
    \label{lemma:1_2}
    In deep networks, could $\exists ~i \in [n-r, n] \quad | \quad \| \nabla_{x_i} f(x^{\text{rel}} + \epsilon^i)\| \gg \| \nabla_{x_j} f(x^{\text{rel}} + \epsilon^j)\|, ~\forall j \in [1, r]$.
\end{lemma}
\begin{proof}
    For the sake of simplicity, lets assume the same network defined in Lemma \ref{lemma:1_1}.

    The derivative of the output with respect to input \( x_i \) is:
    \begin{equation}
        \frac{\partial f}{\partial x_i} = \left( \prod_{l=1}^L w_{li} \right) \cdot \mathbf{1}_{\text{ReLU active at all layers}}
    \end{equation}
    
    Since we assume that \( w_{li} \gg w_{lr}, ~\forall r \in [1, r] \), then:
    \begin{equation}
        \frac{\partial f}{\partial x_i} = \left( \prod_{l=1}^L w_{li} \right) \cdot \mathbf{1}_{\text{ReLU active at all layers}} \gg \frac{\partial f}{\partial x_j} = \left( \prod_{l=1}^L w_{lj} \right) \cdot \mathbf{1}_{\text{ReLU active at all layers}}, \quad \forall j \in [1, r]
    \end{equation}

    This shows that the gradient magnitude is large if all the ReLUs are active and weights are large, despite the input being near-zero, and could obscure the importance of the relevant features during training.

\end{proof}

%\begin{lemma}
%    \label{lemma:1_2}
%    Small irrelevant inputs can cause large output changes
%\end{lemma}
%\begin{proof}
%    This is typical in overparameterized networks where irrelevant features can be assigned large gradients. We analyze the first-order Taylor approximation of the network around a point \( x_0 \):

%   \begin{equation}
%        f(x) \approx f(x_0) + \nabla f(x_0)^T (x - x_0)
%    \end{equation}

%    By Lemma \ref{lemma:1_1} we can assume that there exists:
%    \begin{equation}
%        \|x_{\text{irr}}\| \ll 1, \quad \left\| \frac{\partial f}{\partial x_{\text{irr}}} \right\| \gg \left\| \frac{\partial f}{\partial x_{\text{rel}}} \right\|
%    \end{equation}

%    Then, the change in output is dominated by the irrelevant features:
%    \begin{equation}
%        \left| \nabla f(x_0)^T (x - x_0) \right| \approx \underbrace{\left| \sum_{j=r+1}^n \frac{\partial f}{\partial x_j}(x_0) \cdot (x_j - x_{0j}) \right|}_{irrelevant} \gg \underbrace{\left| \sum_{i=1}^r \frac{\partial f}{\partial x_i}(x_0) \cdot (x_i - x_{0i}) \right|}_{relevant}
%    \end{equation}

%\end{proof}

\subsection{Effect of Small and Irrelevant Inputs when attaching an attention model without feature removal}

\begin{lemma}
    \label{lemma:1_3}
    Given \( g(\cdot): \mathbb{R}^n \to [0, 1]^n \) an attention architecture with no feature removal, where $g(x)_j \approx 0 ~|~ \forall j \in [1, r]$, there could exists some $i \in [n-r, n]$ that can cause large output changes to $f(g(x) \circ x)$ and also large variations in the training procedure given that $\| \nabla_{x_i} f(g(x) \circ x)\| \gg \| \nabla_{x_j} f(g(x) \circ x)\|, ~\forall j \in [1, r]$. 
\end{lemma}
\begin{proof}
    Let \( x \in \mathbb{R}^n \) be the input vector and \( g(x) \in [0, 1]^n \) be a learned gating function, applied element-wise, so that:
    \begin{equation}
        \tilde{x}_i = g(x)_i \cdot x_i
    \end{equation}
    
    Since forall  $x_i \in x_{\text{irr}}, \quad g(x)_i \approx 0$, then $\tilde{x_i} \approx 0$ too. Thus, by Lemma \ref{lemma:1_1}, we know that small values like $\tilde{x_i}$ can cause high variation at the output $\| f(\tilde{x_i}) \| \gg 0$, and the first part of the proof is complete.

    For the training procedure, we are interested in the gradient of the loss \( \mathcal{L} \) with respect to the input \( x_i \). By the chain rule:
    \begin{equation}
        \frac{\partial \mathcal{L}}{\partial x_i}
        = \frac{\partial \mathcal{L}}{\partial \tilde{x}_i} \cdot \frac{\partial \tilde{x}_i}{\partial x_i}
    \end{equation}

    Recall that:
    \begin{equation}
        \tilde{x}_i = g(x)_i \cdot x_i
        \quad \Rightarrow \quad
        \frac{\partial \tilde{x}_i}{\partial x_i} = g(x)_i + x_i \cdot \frac{\partial g(x)_i}{\partial x_i}
    \end{equation}

    Thus, the full gradient becomes:
    \begin{equation}
        \frac{\partial \mathcal{L}}{\partial x_i} = \frac{\partial \mathcal{L}}{\partial \tilde{x}_i} \cdot \left( g(x)_i + x_i \cdot \frac{\partial g(x)_i}{\partial x_i} \right)
    \end{equation}
    Although $g(x)_i \approx 0$, by Lemma \ref{lemma:1_2} we know that $|| \frac{\partial g(x)_i}{\partial x_i} ||$ can be high and affect the results. Thus, the proof is complete.

    This is an important property, because it demonstrates that, even if the attention mask is able to correctly detect the irrelevant features, its contribution to the model during training and inference could still obscure the importance of the relevant features.
\end{proof}

\subsection{Effect of Small and Irrelevant Inputs when attaching DDS}

\begin{theorem}
    Given \( g(\cdot): \mathbb{R}^n \to [0, 1]^n \) an attention architecture with feature removal, and $\mathbf{\Gamma}_M \in \{0, 1\}^n$ a mask that removes the lower values of $g(\cdot) \approx 0$, we can avoid that irrelevant features affect the task performance during training and inference.
\end{theorem}
\begin{proof}
    Let \( x \in \mathbb{R}^n \) be the input vector, $x_i \in x_{\text{irr}}$ an irrelevant feature, and let \( g(x) \in [0, 1]^n \) be a learned gating function, applied element-wise, so that:
    \begin{equation}
        \hat{x}_i = \Gamma_i \cdot g(x)_i \cdot x_i = \Gamma_i \cdot \tilde{x}_i
    \end{equation}
    
    By definition, we have that, since $g(x)_i \approx 0, \forall i \in [n-r, n]$, then 
    \begin{equation}
        \Gamma_i = 0, ~\forall i \in [n-r, n].
    \end{equation}
    
    Consequently:
    \begin{equation}
        \hat{x}_i = 0, \quad i \in [n-r, n],
    \end{equation}
    that is, if the attention mask is able to correctly detect the irrelevant features, their values will be nullified.
    
    Thus, by Lemma \ref{lemma:1_1} we know that:
    \begin{equation}
        f(\hat{x}) \propto \hat{x}_j \cdot \prod_{l=1}^L w_{lj} \gg 0 \cdot \prod_{l=1}^L w_{li} = 0, \quad \forall j \in [1, r],
    \end{equation}
    demonstrating that irrelevant features do not affect the output result.

    Related to the training procedure, by Lemma \ref{lemma:1_3} we have that
    \begin{equation}
        \frac{\partial \mathcal{L}}{\partial x_i} = \Gamma_i \cdot \frac{\partial \mathcal{L}}{\partial \tilde{x}_i} \cdot \left( g(x)_i + x_i \cdot \frac{\partial g(x)_i}{\partial x_i}\right) = 0 \cdot \frac{\partial \mathcal{L}}{\partial \tilde{x}_i} \cdot \left( g(x)_i + x_i \cdot \frac{\partial g(x)_i}{\partial x_i}\right) = 0,
    \end{equation}
    demonstrating that irrelevant features do not affect the gradient of the network, and the proof is complete
\end{proof}

\section{World Models: A Detailed Overview}
\label{appendix:world-models}

In model-based reinforcement learning (RL), the term \emph{world model} typically refers to the combination of two main components that together learn a generative representation of the environment:

\begin{itemize}
    \item \textbf{Vision Model (V):} Compresses each high-dimensional observation (e.g., an image) into a lower-dimensional latent vector \(\mathbf{z}_t\). A common choice for this part is a Variational Autoencoder (VAE), which learns both an encoder and a decoder. The encoder maps images to latent representations, and the decoder is capable of reconstructing or “imagining” frames purely from latent codes.
    \item \textbf{Memory Model (M):} Captures the \emph{temporal} dynamics of the environment in latent space. Typically, this involves training a recurrent model (e.g., an RNN) that outputs a Mixture Density Network (MDN) predicting the next latent vector \(\mathbf{z}_{t+1}\) given the current latent \(\mathbf{z}_t\), the agent’s action \(\mathbf{a}_t\), and the hidden state \(\mathbf{h}_t\). Formally:
    \[
        \mathbf{z}_{t+1} \sim \mathcal{P}(\mathbf{z}_{t+1} \mid \mathbf{z}_t, \mathbf{a}_t, \mathbf{h}_t),
    \]
    where \(\mathcal{P}\) is modeled by the parameters of the mixture components output by the MDN-RNN.
\end{itemize}

Once trained, these two modules—the Vision Model and the Memory Model—constitute the \emph{world model}. They allow an agent to generate sequences of predicted future latent states, which can be decoded back into image space if desired. In essence, the Vision Model provides spatial compression and reconstruction, while the Memory Model predicts how these compressed representations evolve over time, effectively simulating environment dynamics in a more manageable latent space.

\subsection{Controller}
Although crucial for decision-making, the \emph{Controller} (or policy) lies \emph{outside} the world model itself. The Controller uses information from the latent state \(\mathbf{z}_t\) and the RNN hidden state \(\mathbf{h}_t\) to decide which action \(\mathbf{a}_t\) to take. In the specific setting we consider here, the Controller is a simple neural network (without a hidden layer) of the form:
\[
    \mathbf{a}_t = W_c \begin{bmatrix}\mathbf{z}_t \\ \mathbf{h}_t\end{bmatrix} + \mathbf{b}_c,
\]
where \(\mathbf{z}_t \in \mathbb{R}^{d}\) is the current latent state, \(\mathbf{h}_t\) is the hidden state of the RNN, and \(W_c\), \(\mathbf{b}_c\) are learnable parameters. In practice, the Controller can be optimized to maximize returns using a variety of standard algorithms (e.g., evolutionary strategies, policy gradients). In our experiments, it is trained on \emph{real environment} rollouts (i.e., no synthetic data is used to train the Controller).

\subsection{Why World Models?}

\paragraph{Sample Efficiency}
A key motivation behind world models is \emph{sample efficiency}. By learning a generative model of the environment, the goal is to create an internal simulation that is so realistic that the controller can be trained solely within these internal "dreams." This approach would be highly efficient because it eliminates the need to process real images, relying instead on compact latent representations. Developing robust internal models can still yield significant benefits. These include enhanced representations for downstream decision-making and potential improvements in how quickly the agent can learn from real-world samples.

\paragraph{Partial Observability and Model Imperfection}
Because the world model only “sees” what is encoded in the Vision Model’s latent vectors, it may miss unobserved or unmodeled factors that influence the true state. If the Memory Model or the Vision Model are poorly learned (e.g., due to insufficient data or training instability), the latent transitions and reconstructions will deviate from reality. Despite these challenges, well-trained world models often provide a powerful abstraction that can simplify policy learning.

\subsection{Training Procedure}
A conventional workflow for building and using a world model can be summarized as follows:

\begin{enumerate}
    \item \textbf{Data Collection:} Gather trajectories of observations and actions using a random or exploratory policy in the real environment.
    \item \textbf{Train the Vision Model:} Fit an autoencoder (e.g., a VAE) on the collected frames so that each image \(\mathbf{x}_t\) is mapped into a latent vector \(\mathbf{z}_t\), and the model can reconstruct \(\mathbf{x}_t\) from \(\mathbf{z}_t\).
    \item \textbf{Train the Memory Model:} Use latent sequences \(\{(\mathbf{z}_t, \mathbf{a}_t)\}\) to train a recurrent network that outputs mixture density parameters. At each time step \(t\), it predicts a distribution over possible next latents \(\mathbf{z}_{t+1}\).
    \item \textbf{Train the Controller:} Employ the (fixed) Vision and Memory Models to encode real environment observations into latents, and update the Controller’s parameters to maximize an RL objective (e.g., via CMA-ES \cite{DBLP:journals/corr/Hansen16a}).
\end{enumerate}

Because the environment dynamics are approximated by the Memory Model directly in latent space, the agent can generate short- or medium-horizon predictions. In some setups, these predictions can be used for planning or to reduce real-environment interactions. In our setting, the Controller is trained solely with real-environment rollouts, even though \emph{in principle} the world model could be used for additional hypothetical scenarios.

\subsection{Beyond Reinforcement Learning: Real-Time Game Generation} While originally introduced as a model-based RL strategy, world models have recently gained traction in domains \emph{beyond} direct policy learning—particularly in the game-generation community. Here, the world model’s generative capacity is harnessed as a form of \emph{real-time game engine}, where game levels or scenarios can be dynamically created through latent-space rollouts. By sampling how states evolve over time using the Memory Model and then decoding them back to an observable format (e.g., 2D or 3D graphics), developers and researchers can produce procedurally generated worlds that respond to player actions in a highly adaptive manner. This “world model as a game engine” paradigm enables unique forms of content creation and interactive storytelling, blending the boundaries between model-based RL and creative generative applications.

% \subsection{Summary}
% World models~\cite{DBLP:journals/corr/abs-1803-10122} are designed to learn compressed representations of both the visual and temporal aspects of an environment. By coupling a \emph{Vision Model} with a \emph{Memory Model}, they let the agent—or an external Controller—simulate future states in a manageable latent space. These internal simulations can facilitate more efficient policy training and potentially reduce the reliance on expensive real-world interactions, although the exact usage of those simulations varies by experiment and design choice.

\clearpage
\section{Visualization}
\label{appendix:visualization}

This appendix offers visualizations illustrating key operational aspects and the performance of the proposed Dynamic Data Selection (DDS) methodology. These figures showcase DDS's feature selection, learned internal representations, and its impact on reconstruction and generation tasks, complementing the paper's quantitative results.

\subsection{DDS Selection Analysis}
\label{subsec:dds_selection}

Finally, Figure~\ref{fig:dds_mask_percentages_carracing} illustrates how the DDS module behaves at different selection thresholds, demonstrating its ability to identify salient features even at extremely sparse selection rates.

\begin{figure}[!htbp]
\centering
\includegraphics[width=1.0\columnwidth]{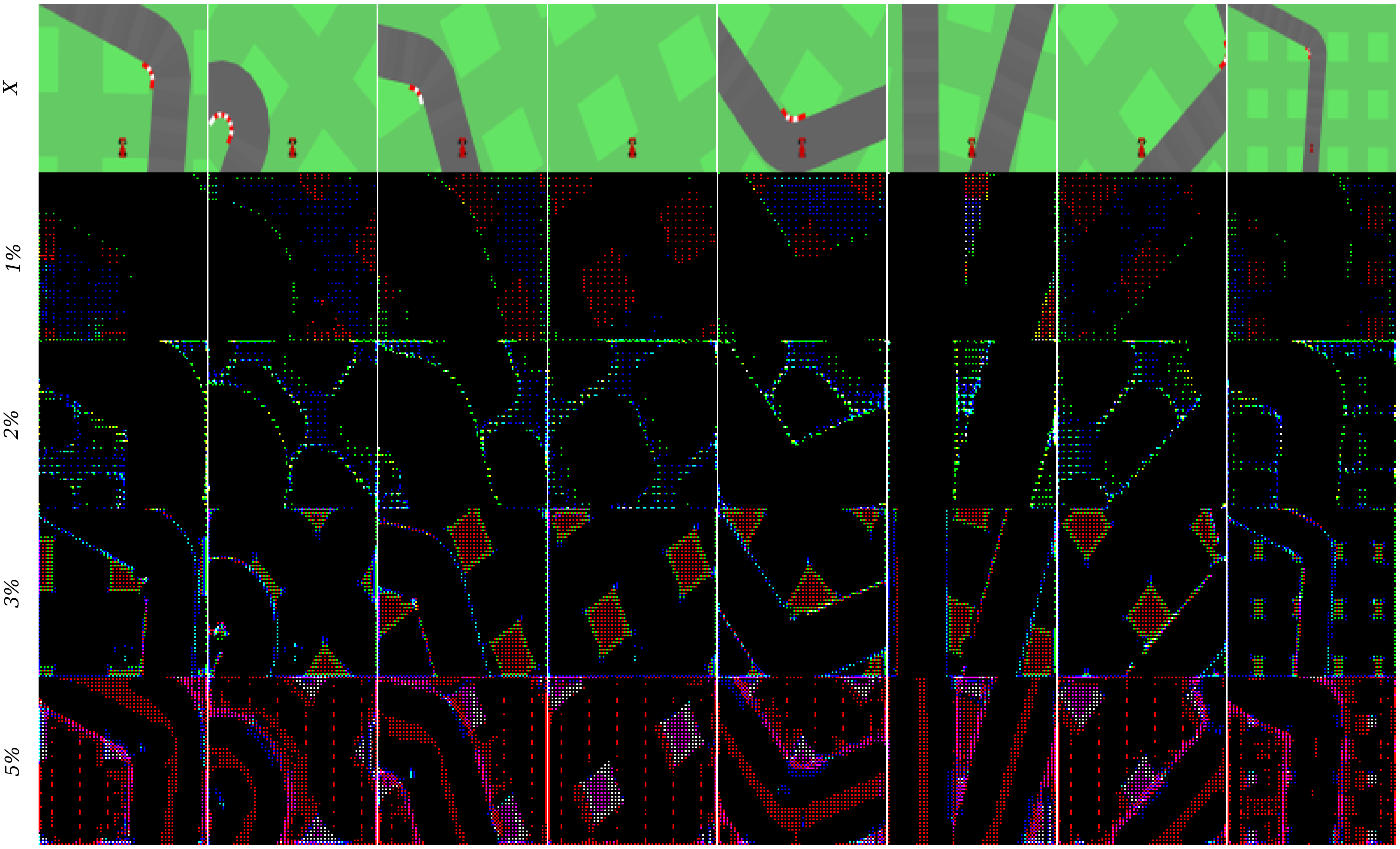}
\caption{Visualization of DDS masks generated at varying selection percentages for the \texttt{CarRacing-v3} environment.
The rows display different input frames.
The columns show (from left to right): the original input ($X$), the output generated with the DDS ($g(\mathbf{X})$) 1\% M, 2\%, 3\% M, 5\% M, and 8\% M.
This figure illustrates the behavior of the DDS module in selecting salient features across different sparsity levels.}
\label{fig:dds_mask_percentages_carracing}
\end{figure}

\subsection{Internal Representations}
\label{subsec:internal_representations}

Figure~\ref{fig:car_internal_representation} and \ref{fig:mario_internal_representation} showcases the intermediate steps in our DDS+VAE model's processing pipeline. These visualizations reveal how the model captures and processes key features within the input data.

\vspace{0.5cm}

\begin{figure}[!htbp]
\centering
\includegraphics[width=1.0\columnwidth]{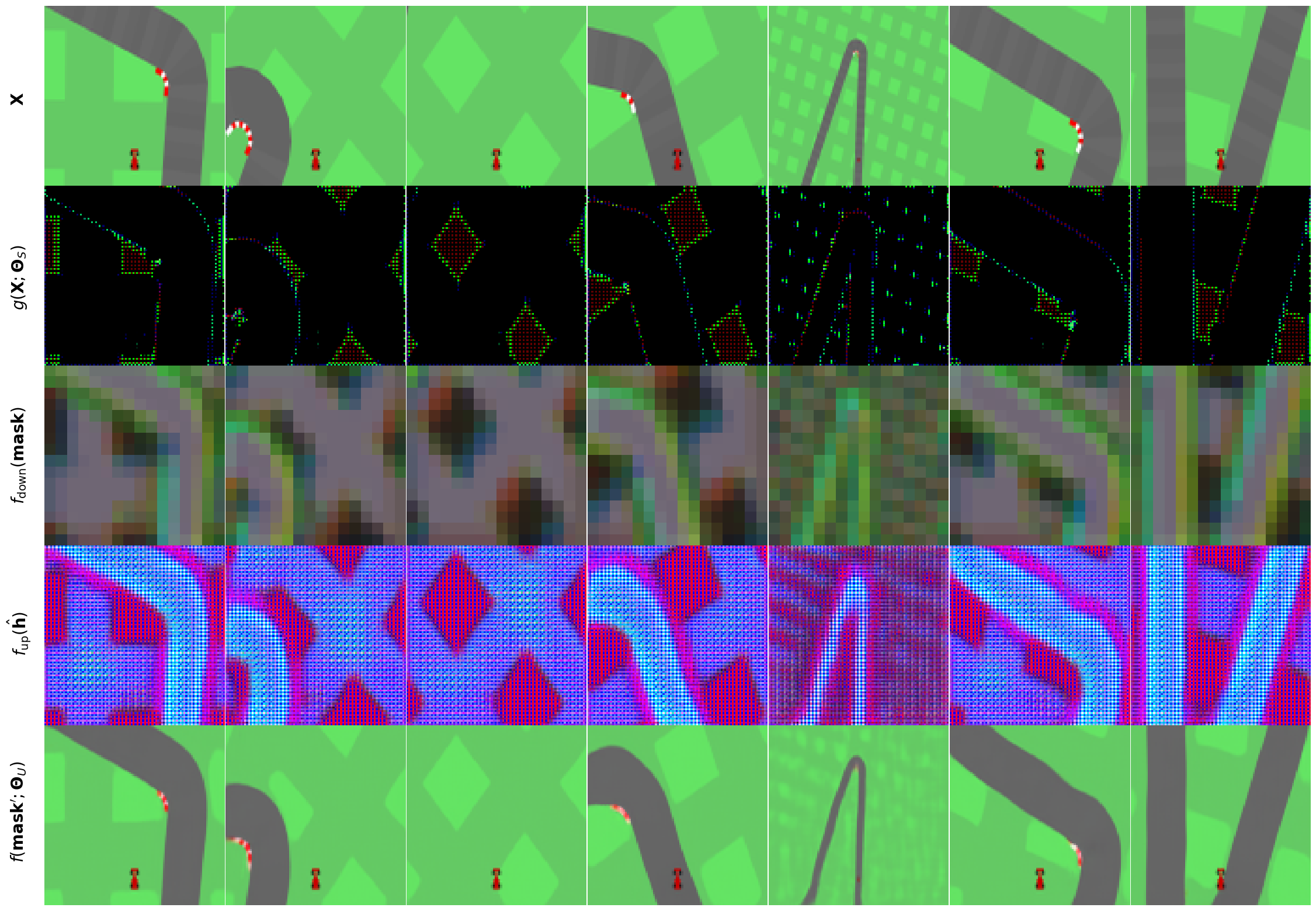}
\caption{Visualization of intermediate stages in our DDS+VAE model trained with $M = 3\%$ selection threshold. The figure illustrates how the model progressively transforms input data through its pipeline. The rows show (from top to bottom): the original input ($x$), the DDS-generated mask ($mask$) with a 3\% selection, the intermediate representation ($h$), the internal representation ($mask'$), and the final reconstructed image ($\hat{x}$).}
\label{fig:car_internal_representation}
\vspace{-0.3cm}
\end{figure}

\begin{figure}[!htbp]
\centering
\includegraphics[width=1.0\columnwidth]{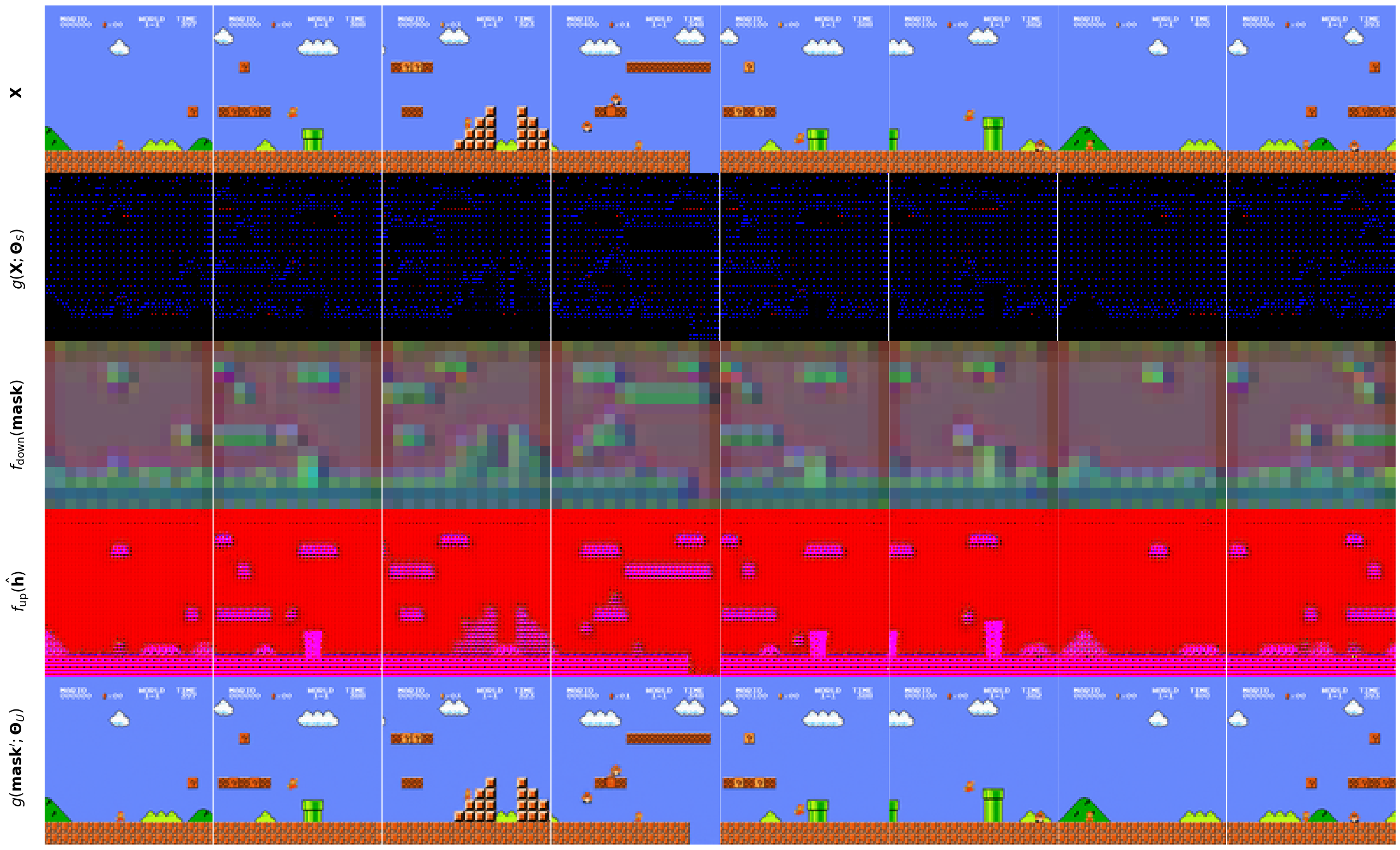}
\caption{Internal representations of the DDS+VAE model for the \texttt{SuperMarioBros-v0} environment. The rows show (from top to bottom): the original input ($x$), the DDS-generated mask ($mask$) with a 3\% selection, the intermediate representation ($h$), the internal representation ($mask'$), and the final reconstructed image ($\hat{x}$).}
\label{fig:mario_internal_representation}
\end{figure}

For a more detailed view of the internal representations, Figure~\ref{fig:dds_vae_internal_representations_carracing_percentages} (on the following page) presents the complete processing pipeline across different DDS selection percentages in the CarRacing-v3 environment.

\clearpage

\begin{figure}[!htbp]
\centering
\includegraphics[width=0.9\columnwidth]{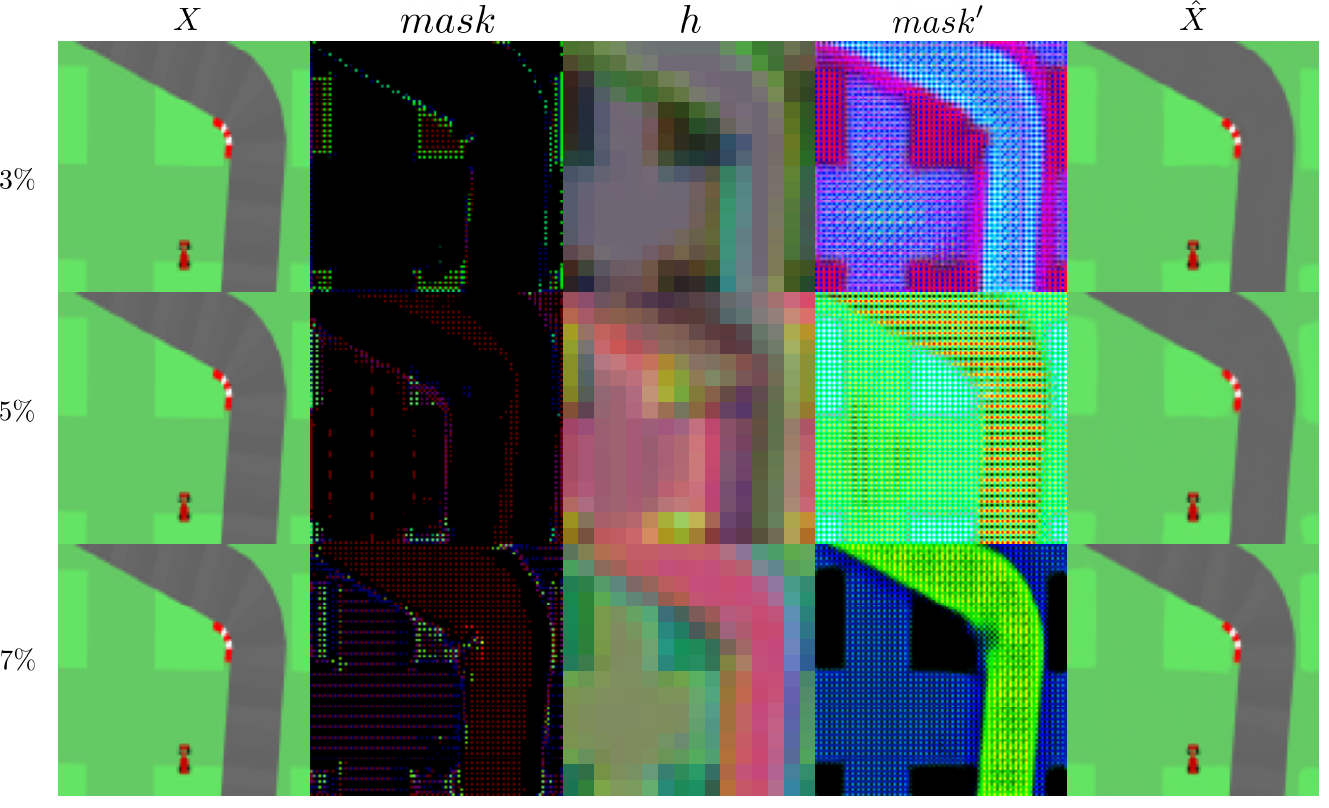}
\caption{Detailed internal representations for the proposed DDS+VAE model on the \texttt{CarRacing-v3} environment.
The columns display (from left to right): original input ($x$), the learned DDS mask ($mask$), the intermediate representation ($h$), an intermediate ($mask'$), and the final reconstruction ($\hat{x}$).
The rows correspond to different DDS selection percentages: 3\% (top row), 5\% (middle row), and 7\% (bottom row).}
\label{fig:dds_vae_internal_representations_carracing_percentages}
\end{figure}

\subsection{Dream Sequence Comparisons}
\label{subsec:dream_sequences}

The ability to generate coherent and detailed dream sequences is crucial for world models. Figure~\ref{fig:dream_seq_mario_comparison} compares dream sequences generated by different models for the \textit{SuperMarioBros-v0} environment. Note that similar results for the \textit{CarRacing-v3} environment are presented in Figure \ref{fig:dream_combined} of the main paper.

\begin{figure}[!htbp]
\centering
\includegraphics[width=1.0\columnwidth]{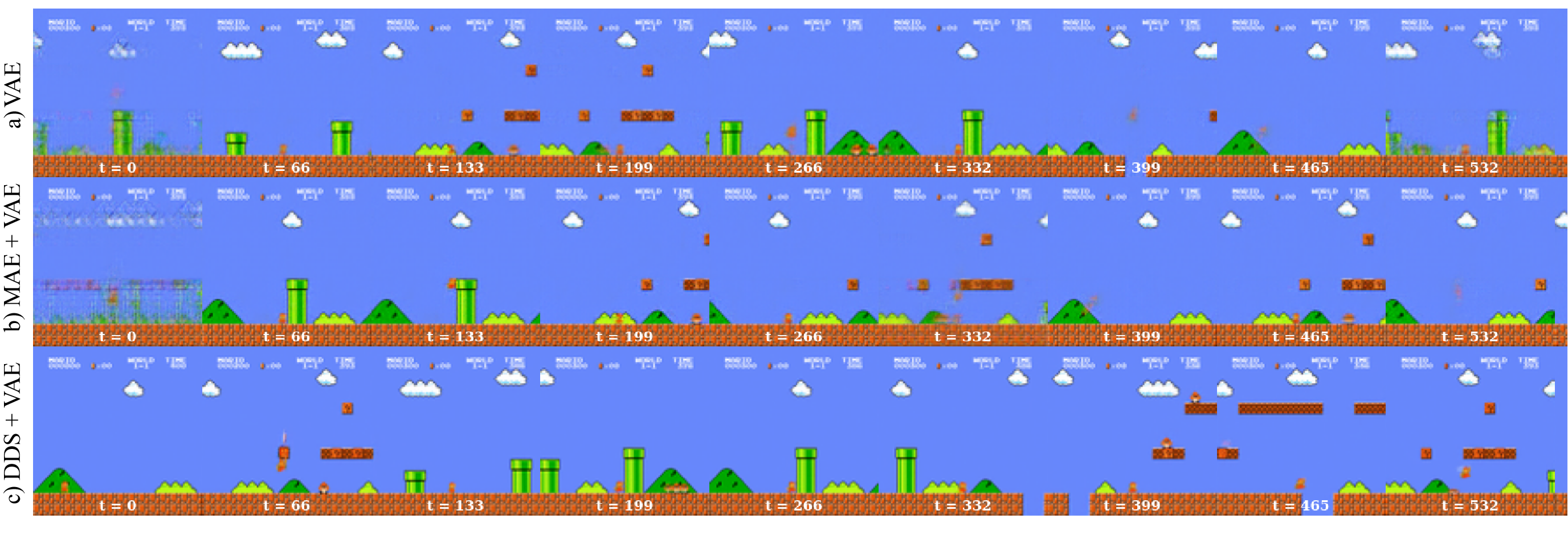}
\caption{Comparison of generated dream sequences for the \texttt{SuperMarioBros-v0} environment.
The rows display sequences from different models: a) standard VAE, b) MAE+VAE, and c) our proposed DDS+VAE.
This visualization highlights the qualitative improvements in sequence coherence and detail achieved with the DDS+VAE model with $M=4\%$ selected features.}
\label{fig:dream_seq_mario_comparison}
\end{figure}

\clearpage

\subsection{Reconstruction Quality Analysis}
\label{subsec:reconstruction_quality}

A key advantage of our DDS+VAE approach is the superior reconstruction quality it achieves. Figure \ref{fig:reconstruction_comparison_car} and \ref{fig:reconstruction_comparison_mario_models} provides a visual comparison between our model and baseline approaches on the \textit{CarRacing-v3} and \textit{SuperMarioBros-v0}	environment.

\begin{figure}[!htbp]
\centering
\includegraphics[width=1.0\columnwidth]{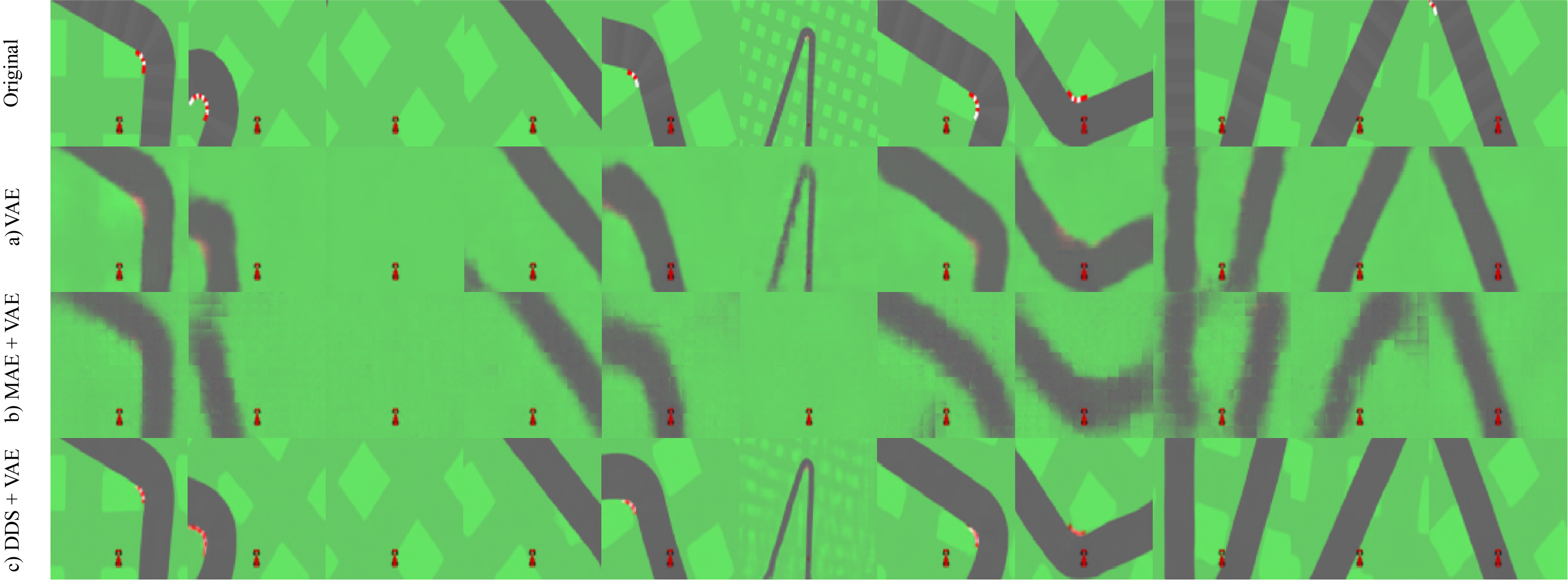}
\caption{Reconstruction quality comparison on the \texttt{CarRacing-v3} environment.
The columns show different frames from the environment.
The rows compare: a) the original input ($x$), b) reconstructions from a standard VAE, c) MAE+VAE model, d) and our proposed DDS+VAE model with $M=4\%$ selected features.
The DDS+VAE model demonstrates superior reconstruction fidelity.}
\label{fig:reconstruction_comparison_car}
\end{figure}

\begin{figure}[!htbp]
\centering
\includegraphics[width=1.0\columnwidth]{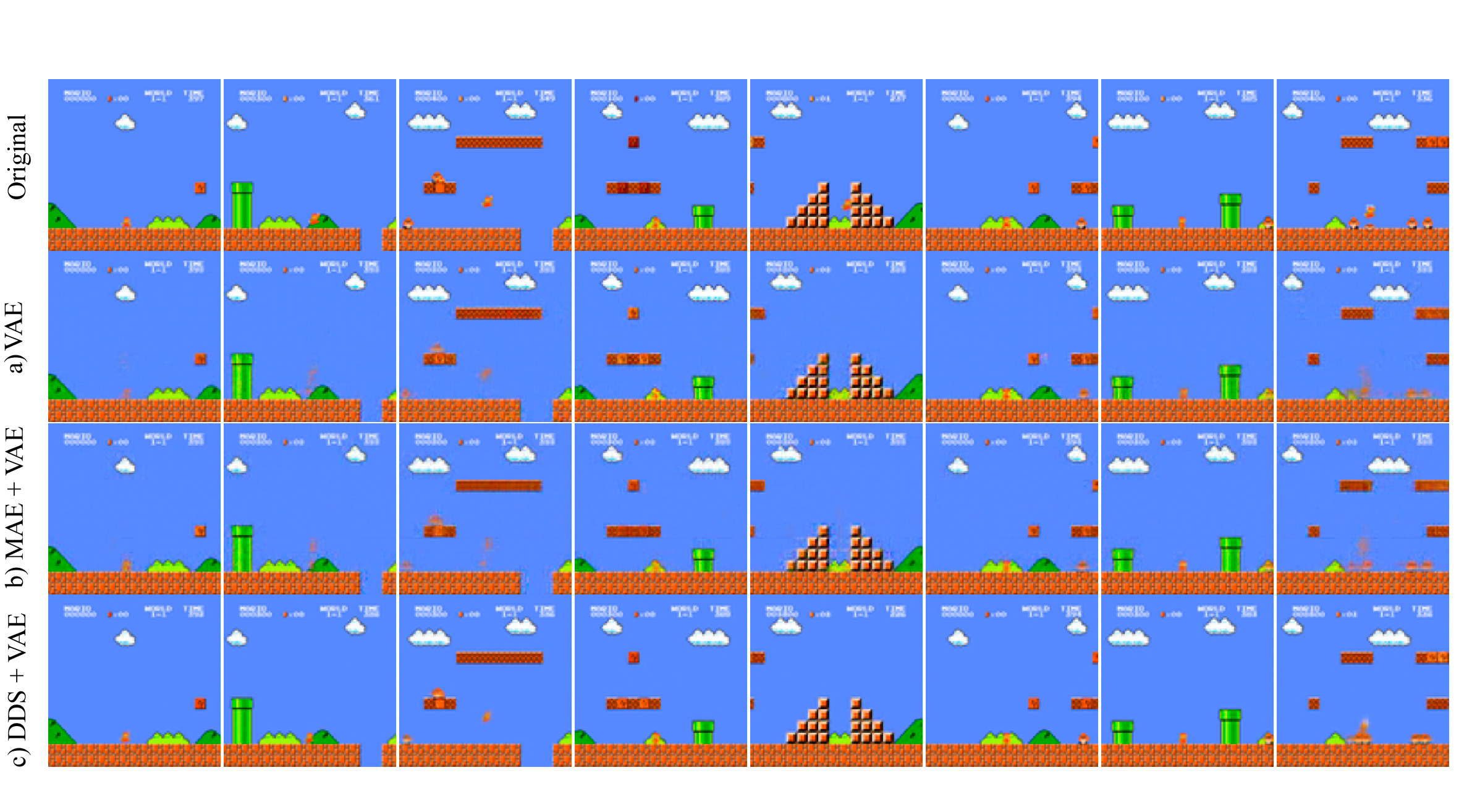}
\caption{Reconstruction quality comparison on the \texttt{SuperMarioBros-v0} environment.
The columns show different frames from the environment. The rows compare: a) the original input ($x$), b) reconstructions from a standard VAE, c) MAE+VAE model, d) and our proposed DDS+VAE model with $M=4\%$ selected features.}
\label{fig:reconstruction_comparison_mario_models}
\end{figure}

\clearpage

% \section{Experimental Setup Details}
% \label{sec:experim_stup}

% \subsection{Hardware Configuration}
% All experiments were conducted on a workstation with the GPU NVIDIA GeForce RTX 3090Ti, featuring 24GB of memory. A 12th Gen Intel(R) Core(TM) i7-12700K, with 20 cores and 64GB RAM memory.

\section{Additional Ablation Studies and Dataset Analysis}
\label{sec:additional_ablations}

This section presents further ablation studies and analyses to provide deeper insights into the DDS framework's behavior and robustness. We explore the impact of core DDS components.

\subsection{Ablation of DDS Core Mechanisms}
\label{subsec:ablation_core_mechanisms}

To verify the importance of the learned selection mechanism and the feature importance scores produced by DDS, we conducted two additional ablation experiments. These experiments evaluate the reconstruction performance on CIFAR-10 against the standard DDS and a naive autoencoder baseline:
\begin{enumerate}
    \item \textbf{Ex1: No Selection Mechanism During Inference:} All feature selection scores are forced to 1 for all pixels, effectively bypassing the dynamic selection.
    \item \textbf{Ex2: Uniform Feature Importance During Inference:} The $M$ selected features are retained, but their learned importance scores are overridden and set to 1.
\end{enumerate}

\begin{table}[h!]
\centering
\begin{tabular}{lccccc}
\toprule
\textbf{Method} & \multicolumn{5}{c}{\textbf{M (Number of Selected Features)}} \\
\cmidrule(lr){2-6}
& \textbf{64} & \textbf{128} & \textbf{256} & \textbf{512} & \textbf{1024} \\
\midrule
Naive AE & 0.018 & 0.012 & 0.008 & 0.005 & 0.004 \\
DDS      & 0.016 & 0.009 & 0.005 & 0.001 & 0.0002 \\
\midrule
Ex1 (No Selection) & 0.2 & 0.2 & 0.2 & 0.2 & 0.2 \\
Ex2 (Uniform Importance) & 0.2 & 0.2 & 0.2 & 0.2 & 0.2 \\
\bottomrule
\end{tabular}
\vspace{1.0em}
\caption{Reconstruction MSE on CIFAR-10 for core DDS mechanism ablations. Both Ex1 and Ex2 demonstrate significantly degraded performance, highlighting the criticality of both the dynamic selection process and the learned feature importance scores for effective reconstruction.}
\label{tab:ablation_core_mech_results}
\end{table}

As anticipated, both Ex1 and Ex2 lead to very poor reconstruction quality. The features masked during training are not learned by the subsequent reconstruction network. Forcing their inclusion or assigning uniform importance at inference time introduces information that the model is not trained to handle, validating DDS's role as an effective feature selector that learns meaningful instance-specific feature importances. 

\subsection{Extended Dataset Analysis: Stanford Cars}
\label{subsec:stanford_cars_analysis}

To assess the generalization of DDS's reconstruction performance and the rationale for selecting $M$ values, we conducted experiments on the Stanford Cars dataset \cite{Yang_2015_CVPR}. This dataset features images with higher resolution (360x240 pixels) compared to CIFAR-10. The chosen $M$ values correspond to the number of features in the latent representation of a Naive Autoencoder when varying its initial channel capacity in powers of two. Table \ref{tab:stanford_cars_results} presents the Mean Absolute Error (MAE) and Mean Squared Error (MSE) for reconstruction.

\begin{table}[h!]
\centering
\begin{tabular}{lcccc}
\toprule
\textbf{M (\% of total features)} & \textbf{Naive MAE} & \textbf{Naive MSE} & \textbf{DDS MAE} & \textbf{DDS MSE} \\
\midrule
3136 (1.2\%)   & 7.99E-02 & 1.33E-02 & 7.48E-02 & 1.18E-02 \\
6272 (2.4\%)   & 6.73E-02 & 1.05E-02 & 3.83E-02 & 3.47E-03 \\
12544 (4.8\%)  & 4.52E-02 & 5.67E-03 & 2.26E-02 & 1.27E-03 \\
25088 (9.7\%)  & 3.55E-02 & 3.47E-03 & 1.34E-02 & 4.91E-04 \\
50176 (19.4\%) & 3.09E-02 & 2.86E-03 & 6.45E-03 & 1.00E-04 \\
\bottomrule
\end{tabular}
\vspace{1.0em}
\caption{Reconstruction MAE and MSE on the Stanford Cars dataset. DDS consistently outperforms the naive autoencoder, and the error reduction pattern mirrors that observed on CIFAR-10: doubling $M$ approximately halves the reconstruction error for DDS, a scaling not achieved by the naive baseline.}
\label{tab:stanford_cars_results}
\end{table}

The results on Stanford Cars corroborate the findings from CIFAR-10. DDS achieves lower reconstruction errors than the naive autoencoder across all tested $M$ values. Notably, as the number of selected features $M$ is doubled, the reconstruction error for DDS is almost halved, demonstrating efficient utilization of the selected features. This contrasts with the naive autoencoder, which exhibits diminishing returns in error reduction with an increasing number of features.

\subsection{Ablation on Input Image Resolution: Stanford Cars}
\label{subsec:image_resolution_ablation}

We further investigated the impact of input image resolution on DDS's reconstruction performance. Using the Stanford Cars dataset, we trained the same network architecture with images resized to 32x32, 64x64, and 128x128 pixels, while varying the percentage of selected features $M$. The results are shown in Table \ref{tab:image_resolution_results}.

\begin{table}[h!]
\centering
\begin{tabular}{lcccc}
\toprule
\textbf{Image Size / M (\%)} & \textbf{3\%} & \textbf{6\%} & \textbf{12\%} & \textbf{24\%} \\
\midrule
32x32   & 0.0102 & 0.0081 & 0.0076 & 0.0068 \\
64x64   & 0.0096 & 0.0075 & 0.0067 & 0.0062 \\
128x128 & 0.0083 & 0.0077 & 0.0063 & 0.0059 \\
\bottomrule
\end{tabular}
\vspace{1.0em}
\caption{Reconstruction MSE on the Stanford Cars dataset for varying input image resolutions and percentages of selected features ($M$).}
\label{tab:image_resolution_results}
\end{table}

The results indicate that larger input images tend to yield lower MSE scores for a given percentage of selected features. This behavior is expected, as higher-resolution images often contain more spatially redundant information (e.g., smooth textures, uniform backgrounds) which are inherently easier for autoencoder-based models to reconstruct efficiently. DDS effectively leverages this by selecting the most informative features even in sparser, lower-resolution inputs, but benefits from the increased reconstructability of larger, more detailed images.

\subsection{Training Time for World Model Experiments}
The World Model experiments involved a two-stage training strategy as detailed in Section \ref{world-models-introduction}. The following table summarizes the number of parameters, training epochs, and computation time for the different models compared in this study. The times reported are for the completion of both training stages. The parameter counts correspond to those mentioned in Section \ref{world-models-introduction}.

\begin{table}[h!]
\centering
\begin{tabular}{lccc}
\toprule
\textbf{Models} & \textbf{Parameters} & \textbf{Epochs (stage1 + stage2)} & \textbf{Time (stage1 + stage2)} \\
\midrule
VAE       & 4,348,547 & 1         & 85 min \\
MAE+VAE   & 6,686,464 & 1 + 2     & 40 min + 60 min \\
DDS+VAE   & 4,039,089 & 1 + 2     & 120 min + 55 min \\
\bottomrule
\end{tabular}
\vspace{1.0em}
\caption{Comparison of different world model vision architectures, detailing their parameter counts, training epochs split by stage, and total training time.}
\label{tab:model_comparison}
\end{table}

The DDS+VAE model, despite its competitive or superior performance in several metrics as discussed in Section \ref{world-models-introduction}, also demonstrates efficiency in terms of parameter count compared to the baselines. The training times provide insight into the computational cost associated with each approach under the described two-stage training regime.

\clearpage

\end{document}